\documentclass{article}

\usepackage{microtype}
\usepackage{graphicx}
\usepackage{subfigure}
\usepackage{booktabs} 

\usepackage{hyperref}
\usepackage{amsmath}
\usepackage{amsthm}
\usepackage{amssymb}
\usepackage{wrapfig}
\usepackage{multirow}
\usepackage{makecell}
\usepackage{array, booktabs, arydshln, xcolor}
\usepackage{mwe}

\usepackage[accepted]{icml2021}

\usepackage{amsmath,amsfonts,bm}

\def\eqref#1{equation~\ref{#1}}
\def\Eqref#1{Equation~\ref{#1}}

\def\1{\bm{1}}

\DeclareMathAlphabet{\mathsfit}{\encodingdefault}{\sfdefault}{m}{sl}
\SetMathAlphabet{\mathsfit}{bold}{\encodingdefault}{\sfdefault}{bx}{n}

\newcommand{\E}{\mathbb{E}}

\DeclareMathOperator*{\argmax}{arg\,max}
\DeclareMathOperator*{\argmin}{arg\,min}

\newcommand{\smallsum}[2]{ {\textstyle 
	\sum\limits_{\scriptscriptstyle #1}^{\scriptscriptstyle #2}} 
}

\newcommand{\smallfrac}[2]{ {\textstyle \frac{#1}{#2}} }
\newcommand{\Set}[1]{\mathcal{#1}}

\usepackage{algorithm}
\usepackage{amsthm}
\usepackage{xspace}
\newcommand{\shorttitle}{UneVEn\xspace}
\newcommand\bs[1]{\boldsymbol{#1}}

\theoremstyle{definition}

\newtheorem*{propositionNO}{Proposition}

\definecolor{lightgray}{rgb}{0.75, 0.75, 0.75}
\definecolor{darkgreen}{rgb}{0, 0.75, 0}
\definecolor{wendelincolor}{rgb}{0.66, 0.33, 0}

\newcommand{\wbr}[2]{{\color{wendelincolor}#1}}

\usepackage[colorinlistoftodos, textsize=tiny, textwidth=20mm]{todonotes}
\definecolor{wbcolor}{rgb}{0.5, 1, 0.5}

\definecolor{cscolor}{rgb}{0.9, 0.0, 0.2}

\icmltitlerunning{UneVEn: Universal Value Exploration for  Multi-Agent Reinforcement Learning}

\begin{document}

\setlength{\abovedisplayskip}{4pt}
\setlength{\belowdisplayskip}{4pt}

\twocolumn[
\icmltitle{UneVEn: Universal Value Exploration for\\ Multi-Agent Reinforcement Learning}



\icmlsetsymbol{equal}{*}

\begin{icmlauthorlist}
\icmlauthor{Tarun Gupta}{toox}
\icmlauthor{Anuj Mahajan}{toox}
\icmlauthor{Bei Peng}{toox}
\icmlauthor{Wendelin B{\"o}hmer}{tode}
\icmlauthor{Shimon Whiteson}{toox}
\end{icmlauthorlist}

\icmlaffiliation{toox}{Department of Computer Science, University of Oxford, Oxford, United Kingdom}
\icmlaffiliation{tode}{Department of Software Technology, Delft University of Technology, Delft, Netherlands}

\icmlcorrespondingauthor{Tarun Gupta}{tarun.gupta@cs.ox.ac.uk}

\icmlkeywords{multi-agent reinforcement learning, deep Q-learning, universal value functions, successor features, relative overgeneralization}

\vskip 0.3in
]



\printAffiliationsAndNotice{}  

\begin{abstract}
VDN and QMIX are two popular value-based algorithms for cooperative MARL that learn a centralized action value function as a monotonic mixing of per-agent utilities. While this enables easy decentralization of the learned policy, the restricted joint action value function can prevent them from solving tasks that require significant coordination between agents at a given timestep. We show that this problem can be overcome by improving the {\em joint exploration} of all agents during training.  Specifically, we propose a novel MARL approach called Universal Value Exploration (\shorttitle) that learns a set of {\em related} tasks simultaneously with a linear decomposition of universal successor features. With the policies of already solved related tasks, the joint exploration process of all agents can be improved to help them achieve better coordination. Empirical results on a set of exploration games, challenging cooperative predator-prey tasks requiring significant coordination among agents, and StarCraft II micromanagement benchmarks show that \shorttitle can solve tasks where other state-of-the-art MARL methods fail. 
\end{abstract}

\section{Introduction}
Learning control policies for cooperative multi-agent reinforcement learning (MARL) remains challenging as agents must search the joint action space, which grows exponentially with the number of agents. Current state-of-the-art value-based methods such as VDN \citep{sunehag2017value} and QMIX \citep{rashid2020monotonic} learn a \textit{centralized} joint action value function as a \textit{monotonic} factorization of \textit{decentralized} agent utility functions. 
Due to this monotonic factorization, the joint action value function can be decentrally maximized as each agent can simply select the action that maximizes its corresponding utility function, known as the Individual Global Maximum principle \citep[IGM,][]{son2019qtran}. 


This monotonic restriction cannot represent the value of all joint actions
and an agent's utility depends on the policies 
of the other agents \citep[nonmonotonicity,][]{mahajan2019maven}. 
Even in collaborative tasks this can exhibit 
{\em relative overgeneralization} 
\citep[RO,][]{panait2006biasing},
when the optimal action's utility falls below that 
of a suboptimal action \citep{wei2018multiagent,wei2019multiagent}.
While this pathology depends in practice 
on the agents' random experiences,
we show in Section \ref{sec:example}
that in expectation RO prevents VDN from learning a large 
set of predator-prey games 
during the critical phase of uniform exploration. 


QTRAN \citep{son2019qtran} and WQMIX \citep{rashid2020weighted} show 
that this problem can be avoided by weighting 
the {\em joint actions of the optimal policy} higher.
They propose to deduce  this weighting from an unrestricted joint value function
that is learned simultaneously.
However, this unrestricted value 
is only a critic of the factored model, 
which itself is prone to RO,
and often fails in practice due to insufficient $\epsilon$-greedy exploration. 
MAVEN \citep{mahajan2019maven} improves exploration by
learning an ensemble of monotonic joint action value functions 
through committed exploration and maximizing 
diversity in the joint team behavior.
However, it does not specifically target optimal actions
and may not work in tasks with strong RO. 

The core idea of this paper is 
that even when a {\em target task} 
exhibits RO under value factorization,
there may be {\em similar tasks} that do not.
If their optimal actions overlap in some states with the target task,
executing these simpler tasks 
can implicitly weight exploration to overcome RO.
We call this novel paradigm Universal Value Exploration (\shorttitle).
To learn different MARL tasks simultaneously, 
we extend Universal Successor Features \citep[USFs,][]{borsa2018universal}
to Multi-Agent USFs (MAUSFs), using a VDN decomposition. 
During execution, \shorttitle samples task descriptions 
from a Gaussian distribution centered around the target
and executes the task with the highest value.
This biases exploration towards the optimal actions 
of tasks that are similar but have already been solved. Sampling these reduces RO for every task that shares the same optimal action and thus increases the number of solvable tasks, eventually overcoming RO for the target task, too. 
This is the different from exploration methods like MAVEN, which keep the task same but reweigh explorative behaviours using the task returns. We show in Section \ref{sec:example} on a classic RO example that \shorttitle can gradually increase the set of {\em solvable tasks} until it includes the target task. 

We evaluate our novel approach against several ablations in predator-prey tasks 
that require significant coordination amongst agents 
and highlight the RO pathology,
as well as other commonly used benchmarks
like StarCraft II \citep{samvelyan2019starcraft}. 
We also show empirically that \shorttitle significantly 
outperforms current state-of-the-art value-based methods 
on target tasks that exhibit strong RO and in zero-shot generalization 
\citep{borsa2018universal} to other reward functions.
\section{Illustrative Example}
\label{sec:example}


\begin{figure}[t!]
	\centering
	\subfigure{\includegraphics[width=.75\columnwidth]{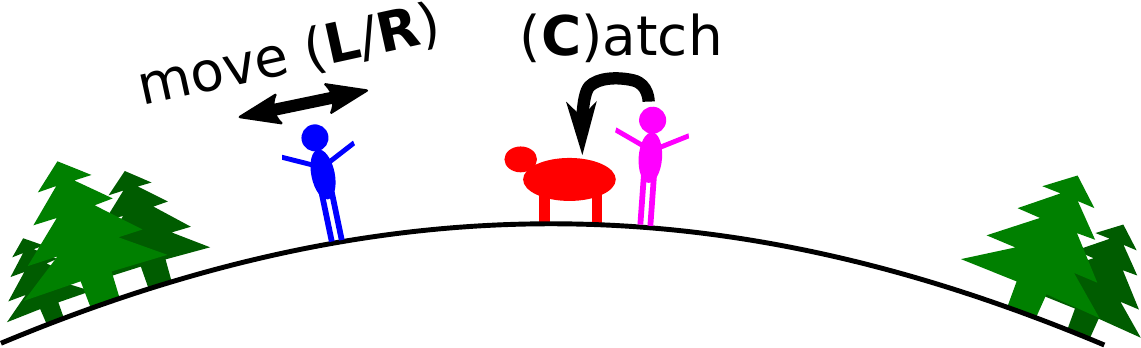}}
	\hfill
	\subfigure{\includegraphics[width=.2\columnwidth]{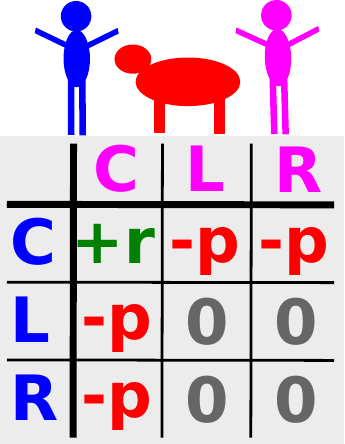}}
	\caption{Simplified predator-prey environment (left),
			where two agents are only rewarded 
			when they both stand next to prey (right).}
	\label{fig:spp_env}
\end{figure}

Figure \ref{fig:spp_env} sketches a simplified predator-prey task where two agents (blue and magenta) can both move left or right ($L/R$) and execute a special `catch' ($C$) action when they both stand next to the stationary prey (red). The agents are collaboratively rewarded (shown on the right of Figure \ref{fig:spp_env}) $+r$ when they catch the prey together 
and punished $-p$ if they attempt it alone, both ending the episode. 
For large $p$, both VDN and QMIX can lead to 
relative overgeneralization \citep[RO,][]{panait2006biasing} 
in the rewarded state $s$ when agent $1$'s utility $Q^1(s, u^1)$
of the catch action $u^1=C$ drops below that of the movement actions $u^1 \in \{L,R\}$. At the beginning of training, 
when the value estimates are near zero 
and both agents explore random actions, we have:
\begin{equation*}
	Q^1(s, C) < Q^1(s, L/R)
	\;\;\Rightarrow\;\; 
	r < p \, \big(\smallfrac{1}{\pi^2(C|s)} - 2\big) + c \,,
\end{equation*}
where $c$ is a constant that depends 
mainly on future values.
See Appendix \ref{sec:app-example} for a formal derivation.
The threshold at which  $p$
yields RO depends strongly on agent $2$'s probability of choosing 
$u^2=C$, i.e.,~$\pi^2(C|s)$. 
For uniform exploration,
this criterion is fulfilled if $p > r$,
but reinforcing other actions than $u^2=C$
can lower this threshold significantly.
However, if agent 2 chooses $u^2=C$ for more than half of its actions,  
no amount of punishment can prevent learning of the correct greedy policy.

Figure \ref{fig:spp_tasks} plots the entire task space w.r.t.~$p$ and $r$,
marking the set of tasks solvable under uniform exploration (green area) 
and tasks exhibiting RO on average for varying $\pi^2(C|s)$.
MAUSFs uses a VDN decomposition of successor features 
to learn all tasks in the black circle simultaneously,
which initially can only solve 
monotonic tasks in the green area.
UneVEn changes this by evaluating random tasks (blue dots), 
and exploring those already solved (magenta dots). 
This increases the fraction of observed $u^2=C$,
and thus the set of solvable tasks,
which eventually reaches the {\em target task} (cross).

\begin{figure}[t!]
	\centering
	\includegraphics[width=.65\columnwidth]{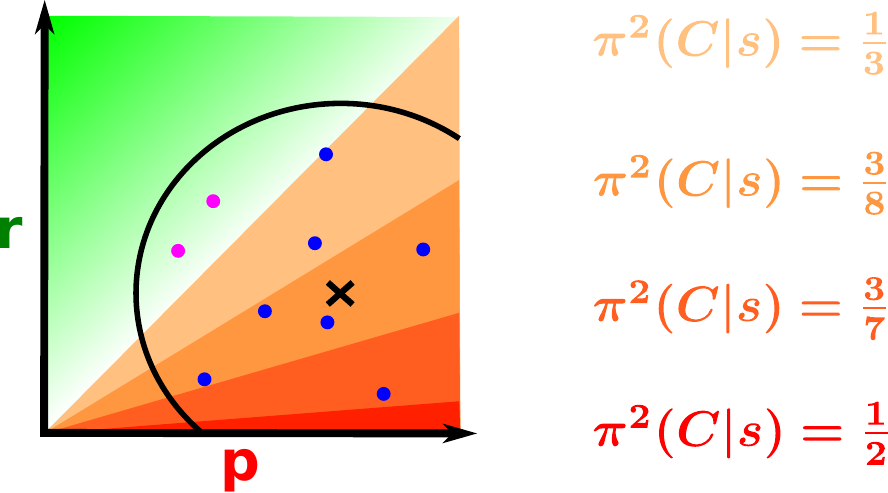}
	\caption{Task space of Figure \ref{fig:spp_env}.
				Tasks solvable under uniform exploration 
				($\pi^2(C|s)=\frac{1}{3}$) are green, 
				shades of red represent tasks that on average
				exhibit RO for different $\pi^2(C|s)$.}
	\label{fig:spp_tasks}
\end{figure}

\section{Background}


\wbr{}{\textbf{Dec-POMDP}:}
A fully cooperative multi-agent task can be formalized as a \textit{decentralized partially observable Markov decision process} \citep[Dec-POMDP,][]{oliehoek2016concise} consisting of a tuple $G=\left\langle \Set S,\Set U, P, R, \Omega, O, n, \gamma \right\rangle$. $s \in \Set S$ describes the true state of the environment. At each time step, each agent $a \in \Set A \equiv \{1,...,n\}$ chooses an action $u^a \in \Set U$, forming a joint action $\bs{u}\in\bs{\Set U}\equiv \Set U^n$. 
This causes a transition in the environment according to the state transition kernel $P(s'|s,\bs{u}):\Set S\times\bs{\Set U}\times \Set S \rightarrow [0,1]$. 
All agents are collaborative and share therefore the same reward function $R(s,\bs{u}): \Set S\times\bs{\Set U}\rightarrow\mathbb{R}$. $\gamma\in[0,1)$ is a discount factor.

Due to \textit{partial observability}, each agent $a$ cannot observe the true state $s$, but receives an observation $o^a \in \Omega$ drawn from observation kernel $o^a \sim O(s,a)$. At time $t$, each agent $a$ has access to its action-observation history $\tau_t^a \in \Set T_t \equiv (\Omega \times \Set U)^t \times \Omega$, on which it conditions a stochastic policy $\pi^a(u_t^a|\tau_t^a)$. $\bs{\tau}_t \in \Set T_t^n$ denotes the histories of all agents. 
The joint stochastic policy $\bs{\pi}(\bs u_t|s_t,\bs\tau_t) \equiv \prod_{a=1}^n \pi^a(u_t^a|\tau_t^a)$ induces a joint action value function : $Q^\pi(s_t, \bs{\tau}_t, \bs{u}_t)=\mathbb{E} \left[G_t|s_t,\bs\tau_t,\bs{u}_t\right]$, where $G_t=\sum^{\infty}_{i=0}\gamma^ir_{t+i}$ is the \textit{discounted return}. 

\textbf{CTDE}: We adopt the framework of \textit{centralized training and decentralized execution} \citep[CTDE,][]{kraemer2016},
which assumes access to all action-observation histories $\bs{\tau}_t$ and global state $s_t$ during training, but each agent's decentralized policy $\pi^a$ can only condition on its own action-observation history $\tau^a$.
This approach can exploit information that is not available during execution and also freely share parameters and gradients, which improves the sample efficiency \citep[see e.g.,][]{foerster2018counterfactual, rashid2020monotonic, bohmer2019deep}.

\textbf{Value Decomposition Networks}: 
A naive way to learn in MARL is \textit{independent Q-learning} \citep[IQL,][]{tan1993multi}, which learns an independent action-value function $Q^a(\tau_t^a, u^a_t; \theta^a)$ for each agent $a$ that conditions only on its local action-observation history $\tau_t^a$. To make better use of other agents' information in CTDE, \textit{value decomposition networks} \citep[VDN,][]{sunehag2017value} represent the joint action value function $Q_{tot}$ as a sum of per-agent {\em utility functions} $Q^a$: $Q_{tot}(\bs{\tau}, \bs{u}; \theta) \equiv 
	\sum_{a=1}^n Q^a(\tau^a, u^a; \theta)$. 
Each $Q^a$ still conditions only on individual action-observation histories and can be represented by an agent network that shares parameters across all agents.
The joint action-value function $Q_{tot}$ can be trained 
using Deep Q-Networks \citep[DQN,][]{mnih2015human}.  Unlike VDN, QMIX \citep{rashid2020monotonic} represents the joint action-value function $Q_{tot}$ with a nonlinear \textit{monotonic} combination of individual utility functions. The greedy joint action in both VDN and QMIX  
can be computed in a decentralized fashion by individually maximizing each agent's utility. See \citet{oroojlooyjadid2019review} for a more in-depth overview of cooperative deep MARL. 

\textbf{Task based Universal Value Functions}: In this paper, we consider tasks that differ only in their reward functions $R_{\bs {w}}(s, \bs u) \equiv \bs w^\top \bs\phi(s, \bs u)$, which are linear combinations of a set of basis functions $\bs \phi : \Set S \times \bs{\Set U} \to \mathbb R^d$. Intuitively, the basis functions $\bs \phi$ encode potentially rewarded events, such as opening a door or picking up an object.  We use the  weight vector $\bs{w}$ to denote the task with reward function $R_{\bs{w}}$. Universal Value Functions \citep[UVFs,][]{schaul2015universal} extend DQN to learn a \textit{generalizable} value function conditioned on tasks. UVFs are typically of the form $Q^\pi(s_t, \bs{u}_t, \bs{w})$ to indicate the action-value function of task $\bs{w}$ under policy $\pi$ at time $t$ as:
\def\eqspace{\;\;}
\begin{align} \label{uvf-q}
	Q^\pi(s_t, \bs u_t, \bs{w}) &= \mathbb{E}^\pi \big[ 
		\,{\textstyle\sum\limits_{i=0}^\infty} 
			\gamma^i \, R_{\bs w}(s_{t+i}, \bs u_{t+i}) \,\big|\, s_t, \bs u_t \big] \nonumber \\&\!\!\!\!= \mathbb{E}^\pi \big[ 
		\,{\textstyle\sum\limits_{i=0}^\infty} 
			\gamma^i \, \bs{\phi}(s_{t+i}, \bs u_{t+i})^\top \bs{w} \,\big|\, s_t, \bs u_t \big].
\end{align}

\textbf{Successor Features}: The Successor Representation \citep{dayan1993improving} has been widely used in single-agent settings \citep{barreto2017successor, barreto2018transfer, borsa2018universal} to generalize across tasks with given reward specifications.
By simply rewriting the definition of the action value function $Q^\pi(s_t, \bs{u}_t, \bs w)$ of task $\bs w$ from \Eqref{uvf-q} we have:
\def\eqspace{\;\;}
\begin{align} \label{single-sf}
	Q^\pi(s_t, \bs u_t, \bs{w}) 
	&= \mathbb{E}^{\pi} 
		\big[ \,{\textstyle\sum\limits_{i=0}^\infty} 
			\gamma^i \, \bs{\phi}(s_{t+i}, \bs u_{t+i}) 
		\, \big| \, s_t, \bs u_t\big]^\top \bs{w} \nonumber \\
	&\equiv \bs{\psi}^\pi\!(s_t, \bs u_t)^\top \bs{w} \,,
\end{align}
where $\bs{\psi}^\pi\!(s, \bs u)$ are the Successor Features (SFs) under policy $\pi$. For the optimal policy $\pi^{\star}_{\bs{z}}$ of task $\bs{z}$, the SFs $\bs{\psi}^{\pi^{\star}_{\bs{z}}}$ summarize the dynamics 
under this policy, which can then be weighted with any reward vector $\bs{w} \in \mathbb{R}^d$ to instantly evaluate policy $\pi^{\star}_{\bs{z}}$ on it:
$Q^{\pi^{\star}_{\bs{z}}}(s, \bs u, \bs{w}) = 
	\bs{\psi}^{\pi^{\star}_{\bs{z}}}(s, \bs u)^\top \bs{w}.$

\textbf{Universal Successor Features and Generalized Policy Improvement}: \citet{borsa2018universal} introduce universal successor features (USFs) that learn SFs conditioned on tasks using the \textit{generalization} power of UVFs. Specifically, they define UVFs of the form $Q(s, \bs{u}, \bs{z}, \bs{w})$ which represents the value function of policy $\pi_{\bs{z}}$ evaluated on task $\bs{w} \in \mathbb{R}^d$. These UVFs can be factored using the SFs property (\Eqref{single-sf}) as: $Q(s, \bs{u}, \bs{z}, \bs{w}) = \bs{\psi}(s, \bs{u}, \bs{z})^\top \bs{w}$, where $\bs{\psi}(s, \bs{u}, \bs{z})$ are the USFs that generate the SFs induced by task-specific policy $\pi_{\bs{z}}$. One major advantage of using SFs is the ability to \textit{efficiently} do generalized policy improvement \citep[GPI,][]{barreto2017successor}, which allows a new policy to be computed for \textit{any unseen} task based on instant policy evaluation of a \textit{set} of policies on that unseen task with a simple dot-product. Formally, given a set $\mathcal{C} \subseteq \mathbb{R}^d$ of tasks and their corresponding SFs $\{\bs{\psi}(s, \bs{u}, \bs{z})\}_{\bs{z} \in \mathcal{C}}$ induced by corresponding policies $\{\pi_{\bs{z}}\}_{\bs{z} \in \mathcal{C}}$, a new policy $\pi'_{\bs w}$ for any unseen task $\bs{w} \in \mathbb{R}^d$ can be derived using:
\begin{align} \label{usf-gpi}
	\pi'_{\bs w}(s) &\in \argmax_{\bs{u} \in \bs{\Set U}} 
	\max_{\bs{z} \in \mathcal{C}} Q(s, \bs{u}, \bs{z}, \bs{w}) \nonumber \\
	&\in
	\argmax_{\bs{u} \in \bs{\Set U}} \max_{\bs{z} \in \mathcal{C}} 
	\bs{\psi}(s, \bs{u}, \bs{z})^\top \bs{w}.
\end{align} 
Setting $\mathcal{C} = \{ \bs{w}\}$ allows us to revert back to UVFs, as we evaluate SFs induced by policy $\pi_{\bs{w}}$ on task $\bs{w}$ itself. However, we can use any set of tasks that are similar to $\bs{w}$ based on some similarity distribution $\mathcal{D}(\cdot | \bs{w})$. The computed policy $\pi'_{\bs w}$ is guaranteed to perform no worse on task $\bs{w}$ than \textit{each} of the policies $\{\pi_{\bs{z}}\}_{\bs{z} \in \mathcal{C}}$ \citep{barreto2017successor}, but often performs much better. SFs thus enable efficient use of GPI, which allows \textit{reuse} of learned knowledge for zero-shot generalization.

\section{Multi-Agent Universal Successor Features}

\begin{figure*}[t!]
	\includegraphics[width=0.97\hsize]{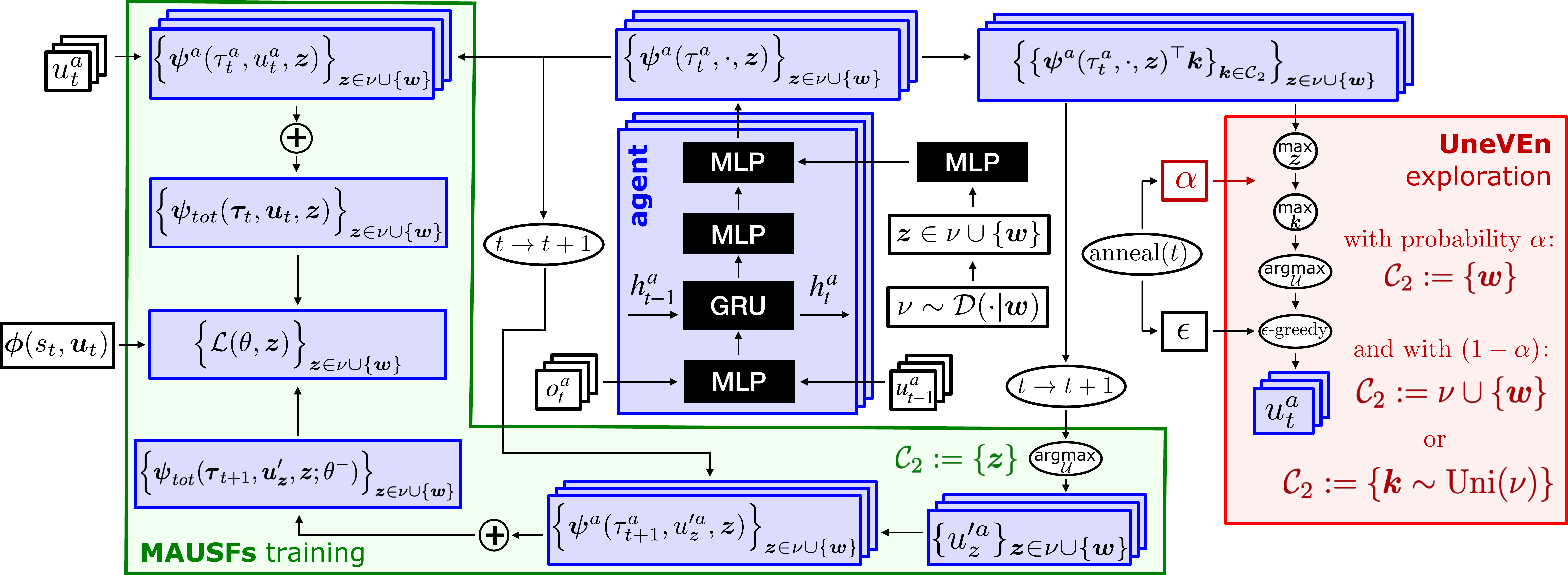}
	\caption{Schematic illustration of the MAUSFs training and \shorttitle exploration with GPI policy.}
	\label{fig:arch}
\end{figure*}

In this section, we introduce Multi-Agent Universal Successor Features (MAUSFs), extending single-agent USFs \citep{borsa2018universal} to multi-agent settings and show how we can learn generalized \textit{decentralized} greedy policies for agents. The USFs based centralized joint action value function $Q_{tot}(\bs{\tau}, \bs{u}, \bs{z}, \bs{w})$ allows evaluation of joint policy $\bs{\pi}_{\bs{z}}  = \left\langle\pi^1_{\bs z}, \ldots, \pi^n_{\bs z}\right\rangle$ comprised of local agent policies $\pi^a_{\bs z}$ of the \textit{same} task $\bs{z}$ on task $\bs{w}$. However, each agent $a$ may execute a different policy $\pi^a_{\bs z^a}$ of different task $\bs z^a \in \mathcal C$, resulting in a combinatorial set of joint policies. 
Maximizing over all combinations  $\bar{\bs z} \equiv \langle \bs z^1, \ldots, \bs z^n \rangle \in \mathcal C^n$ should therefore enormously improve GPI. 
To enable this flexibility, we define the joint action-value function ($Q_{tot}$) of joint policy $\bs{\pi}_{\bar{\bs{z}}} = \{\pi^a_{\bs z^a}\}_{\bs z^a \in \mathcal C}$ evaluated on any task $\bs{w} \in \mathbb{R}^d$ as: 
$Q_{tot}(\bs{\tau}, \bs{u}, \bar{\bs{z}}, \bs{w}) = \bs{\psi}_{tot}(\bs{\tau}, \bs{u}, \bar{\bs{z}})^\top \bs{w}$, where $\bs{\psi}_{tot}(\bs{\tau}, \bs{u}, \bar{\bs{z}})$ are the MAUSFs of $(\bs{\tau}, \bs{u})$ summarizing the joint dynamics of the environment under joint policy $\bs{\pi}_{\bar{\bs{z}}}$. However, training centralized MAUSFs and using centralized GPI to maximize over a combinatorial space of $\bar{\bs{z}}$ becomes impractical when there are more than a handful of agents, since the joint action space ($\bs{\mathcal U}$) and joint task space ($\mathcal{C}^n$) of the agents grows exponentially with the number of agents. To leverage CTDE and enable decentralized execution by agents, we therefore propose novel \textit{agent-specific SFs} for each agent $a$ following local policy $\pi_{\bs{z}^a}^a$, which condition only on its own local action-observation history and task $\bs{z}^a$. 

\textbf{Decentralized Execution}: 
We define local \textit{utility} functions for each agent $a$ as $Q^a(\tau^a, u^a, \bs{z}^a, \bs{w}) = \bs{\psi}^a(\tau^a, u^a, \bs{z}^a; \theta)^\top \bs{w}$, where $\bs{\psi}^a(\tau^a, u^a, \bs{z}^a; \theta)$ are the local agent-specific SFs induced by local policy $\pi^a_{\bs{z}^a}(u^a | \tau^a)$ of agent $a$ sharing parameters $\theta$. Intuitively, $Q^a(\tau^a, u^a, \bs{z}^a, \bs{w})$ is the utility function for agent $a$ when local policy $\pi^a_{\bs{z}^a}(u^a | \tau^a)$ of task $\bs{z}^a$ is executed on task $\bs{w}$. We use VDN decomposition to represent MAUSFs $\bs{\psi}_{tot}$ as a sum of local agent-specific SFs for each agent $a$:
\begin{align} \label{mausf-def}
	Q_{tot}(\bs{\tau}, \bs{u}, \bar{\bs{z}}, \bs{w}) 
	&= {\textstyle\sum\limits_{a=1}^{n}} Q^a(\tau^a\!, u^a\!, \bs{z}^a\!, \bs{w}) \nonumber \\
	&= 
		\underbrace{{\textstyle\sum\limits_{a=1}^{n}} 
			\bs{\psi}^a(\tau^a\!, u^a\!, \bs{z}^a; \theta)^\top}_{\bs{\psi}_{tot}(\bs \tau, \bs u, \bar{\bs{z}}; \theta)}
		\bs{w} 
\end{align}
%
 We can now learn local agent-specific SFs $\bs{\psi}^a$ for each agent $a$ that can be instantly weighted with any task vector $\bs{w} \in \mathbb{R}^d$ to generate local utility functions $Q^a$, thereby allowing agents to use the GPI policy in a decentralized manner.

\textbf{Decentralized Local GPI}: Our novel agent-specific SFs allow each agent $a$ to locally perform decentralized GPI by instant policy evaluation of a set $\mathcal C$ of local task policies $\{\pi^a_{\bs z^a}\}_{\bs z^a \in \mathcal C}$ on any unseen task $\bs w$ to compute a local GPI policy. Due to the linearity of the VDN decomposition, this is equivalent to maximization over all combinations of $\bar{\bs z} \equiv \langle\bs z^1, \ldots, \bs z^n\rangle \in \mathcal C \times \ldots \times \mathcal C \equiv \mathcal C^n$ as:
%
\def\eqgap{\;\;}
\begin{align}
	\bs{\pi}'_{\bs w}(\bs \tau) &\in
	 \argmax_{\bs u \in \bs{\mathcal U}} \max_{\bar{\bs z} \in \mathcal C^n} 
	Q_{tot}(\bs\tau, \bs u, \bar{\bs z}, \bs w) \nonumber \\
	&\!\in\! 
	\big\{\! \argmax_{u^a \in \mathcal U} \max_{\bs z^a \in \mathcal C} 
	\bs{\psi}^a(\tau^a, u^a, \bs{z}^a; \theta)^\top \bs w \big\}_{a=1}^n \,.
\end{align}
As all of the above relies on the linearity of the VDN decomposition, it cannot be directly applied to nonlinear mixing techniques like QMIX \citep{rashid2020monotonic}.  

\textbf{Training}: MAUSFs for task combination $\bar{\bs{z}}$ 
are trained end-to-end by gradient descent on the loss: 
\def\eqgap{\;\;}
\begin{align} \label{sf_loss}
	\mathcal{L}(\theta, \bar{\bs z}) = 
	\mathbb E_{\sim \mathcal{B}} \Big[ \big\|
	\bs{\phi}(s_{t}, \bs{u}_{t}) &+ \gamma 
		\bs{\psi}_{tot}(\bs\tau_{t+1}, \bs{u'}_{\bar{\bs{z}}}, \bar{\bs z}; \theta^{-}) \nonumber \\
	&- \bs{\psi}_{tot}(\bs\tau_t, \bs u_t, \bar{\bs z}; \theta) \big\|^2_2\Big] ,
\end{align}
where the expectation is over a minibatch of samples $\{(s_t, \bs{u}_{t}, \bs{\tau}_{t})\}$ from the replay buffer $\mathcal{B}$ \citep{lin1992},
$\theta^-$ denotes the parameters 
of a target network \citep{mnih2015human} 
and joint actions $\bs{u'}_{\bar{\bs{z}}}= \{u'^a_{\bs{z}^a}\}_{a=1}^n$ 
are selected individually by each agent network 
using the current parameters $\theta$ 
\citep[called Double $Q$-learning,][]{hasselt2016}:
$u'^a_{\bs{z}^a} = \argmax_{u \in \mathcal{U}} 
\bs{\psi}^a(\tau_{t+1}^a, u, \bs{z}^a; \theta)^\top \bs{z}^a$.
Each agent learns therefore local agent-specific 
SFs $\bs{\psi}^a(\tau^a, u^a, \bs{z}; \theta)$
by gradient descent on $\mathcal L(\theta, \bar{\bs z})$ 
for all $\bs z \in \Set C \equiv \nu \cup \{\bm w\}$,
where $\nu \sim \Set D(\cdot|\bs w)$ is drawn 
from a distance measure around target task $\bm w$. 
The green region of Figure \ref{fig:arch} shows a CTDE based architecture to train MAUSFs for a given target task $\bs{w}$. A detailed algorithm is presented in Appendix \ref{app-a}.

\section{\shorttitle} \label{schemes}
In this section, we present \shorttitle (red region of Figure \ref{fig:arch}),  which leverages MAUSFs and decentralized GPI to help overcome relative overgeneralization on the target task $\bs{w}$. At the beginning of every exploration episode, we sample a set of related tasks $\nu = \{ \bs{z} \sim \mathcal{D}(\cdot | \bs{w})\}$, containing potentially simpler reward functions, from a distribution $ \mathcal{D}$ around the target task. The basic idea is that {\em some} of these related tasks can be efficiently learned with a factored value function. 
These tasks are therefore solved early and exploration 
concentrates on the state-action pairs that are useful to them. If other tasks close to those already solved have the same optimal actions,
this implicit weighting allows to solve them too
\citep[shown by][]{rashid2020weighted}.
Furthermore, tasks closer to $\bs w$ are sampled more frequently,
which biases the process to eventually overcome
relative overgeneralization on the target task itself. 


%
Our method assumes that the basis functions $\phi$ and the reward-weights $\bs{w}$ of the target task are known, but \citet{barreto2020fast} show that both can be learned using multi-task regression. Many choices for $\mathcal D$ are possible, but in the following we sample related tasks using a normal distribution centered around the target task $\bs{w} \in \mathbb{R}^d$ with a small variance $\sigma$ as $\mathcal{D} = \mathcal{N}(\bs{w}, \sigma \textbf{I}_d)$. This samples similar tasks closer to $\bs{w}$ more frequently. As long as $\sigma$ is large enough to cover tasks that do not induce RO (see Figure \ref{fig:spp_tasks}), our method works well and therefore, does not rely on any domain knowledge. The resulting task vectors weight the basis functions $\bs\phi$ differently and represent different reward functions. In particular the varied reward functions can make these tasks much easier, but also harder, to solve with monotonic value functions.
The consequences of sampling harder tasks on learning are discussed with the below action-selection schemes.


\textbf{Action-Selection Schemes}: \shorttitle uses two novel schemes to enable action selection based on related tasks. To emphasize the importance of the target task, we define a probability $\alpha$ of selecting actions based on the target task. Therefore, with probability $1 - \alpha$, the action is selected based on the related task. Similar to other exploration schemes, $\alpha$ is annealed from 0.3 to 1.0 in our experiments over a fixed number of steps at the beginning of training. Once this exploration stage is finished (i.e., $\alpha = 1$), actions are always taken based on the target task's joint action value function. Each action-selection scheme employs a local decentralized GPI policy, which maximizes over a set of policies $\pi_z$ based on
$\bs{z} \in \mathcal{C}_1$ (also referred to as the \textit{evaluation} set) to estimate the $Q$-values of another 
set of tasks $\bs{k} \in \mathcal{C}_2$ (also referred to as the \textit{target}  set)  using:
\begin{align} \label{scheme-gpi}
	\bs{u}_{t} =  \Big\{ u_{t}^a = \argmax_{u \in \mathcal{U}}
		\max_{\bs{k} \in \mathcal{C}_2} \max_{\bs{z} \in \mathcal{C}_1} 
		\overbrace{\bs{\psi}^a(\tau^a_t, u, \bs{z}; \theta)^\top \bs{k}
			}^{Q^a(\tau^a_t, u, \bs{z}, \bs{k})}
		\Big\}_{a \in \mathcal{A}}.
\end{align} 
Here $\mathcal{C}_1 = \nu \cup \{\bs{w}\}$ is the set of target and related tasks that induce the policies that are evaluated (dot-product) on the set of tasks $\mathcal{C}_2$, which varies with different action-selection schemes. The red box in Figure \ref{fig:arch} illustrates \shorttitle exploration. For example, $Q$-learning always picks actions based on the target task, i.e., the target set $\mathcal{C}_2= \{\bs{w}\}$. However, this scheme does not favour important joint actions. 
We call this default action-selection scheme \textbf{target GPI} and execute it with probability $\alpha$.
We now propose two novel action-selection schemes based on related tasks with probability $ 1-\alpha$, and thereby implicitly weighting joint actions during learning. 


\textbf{Uniform GPI}: At the beginning of each episode, this action-selection scheme uniformly picks \textit{one} related task, i.e., the target set $\mathcal{C}_2 = \{\bs{k} \sim \text{Uniform}(\nu)\}$, and selects actions based on that task using the GPI policy throughout the episode. This uniform task selection explores the learned policies of all related tasks in $\mathcal D$. This works well in practice as 
there are often enough simpler tasks to induce the required bias over important joint actions. However, if the sampled related task is harder than the target task, the action-selection based on these harder tasks might hurt learning on the target task and lead to higher variance during training.

\textbf{Greedy GPI}: At every time-step $t$, this action-selection scheme picks the task $\bs{k} \in \nu \cup \{\bs{w}\}$ that gives the highest $Q$-value amongst the related and target tasks, i.e., the target set becomes $\mathcal{C}_2 = \nu \cup \{\bs{w}\}$. Due to the greedy nature of this action-selection scheme, exploration is biased towards solved tasks, as those have larger values. We are thus exploring the solutions of tasks  that are both solvable and similar to the target task $\bs w$,
which should at least in some states lead to the 
same {\em optimal joint actions} as $\bs w$.

\textbf{NO-GPI}: To demonstrate the influence of GPI on the above schemes,
we also investigate ablations, where we define the evaluation set $\mathcal{C}_1 = \{\bs{k}\}$ to only contain the currently estimated task $\bs k$, i.e., using $\bs{u}_{t} =  \{ u_{t}^a = \argmax_{u \in \mathcal{U}}\max_{\bs{k} \in \mathcal{C}_2} \bs{\psi}^a(\tau^a_t, u, \bs{k}; \theta)^\top \bs{k}\}_{a \in \mathcal{A}}$ for action selection.

\section{Experiments}
\label{sec:experiments}
In this section, we evaluate \shorttitle on a variety of complex domains. For evaluation, all experiments are carried out with six random seeds and results are shown with $\pm$ standard error across seeds.  We compare our method against a number of SOTA value-based MARL approaches: IQL \citep{tan1993multi}, VDN \citep{sunehag2017value}, QMIX \citep{rashid2020monotonic}, MAVEN \citep{mahajan2019maven}, WQMIX \citep{rashid2020weighted}, QTRAN \citep{son2019qtran}, and QPLEX \citep{wang2020qplex}.

\textbf{Domain 1 : $m$-Step Matrix Game}

\begin{figure}[b!]
	\centering
	\includegraphics[width=.9\columnwidth]{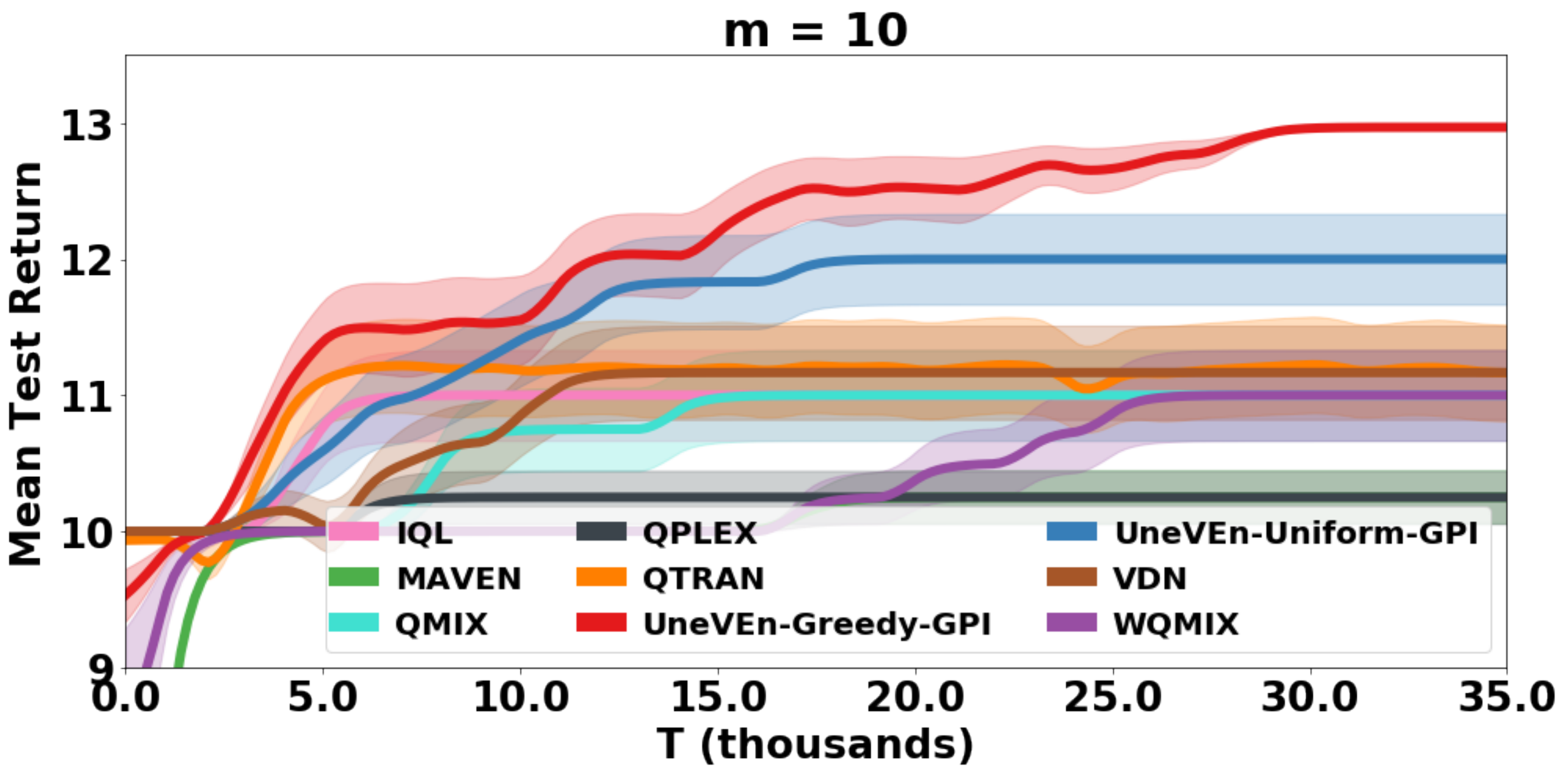}
	\caption{SOTA comparison on the $m$-step matrix game ($m = 10$).}
	\label{mstep:all}
\end{figure}

\begin{figure*}[htbp]
\centering     
\subfigure{\includegraphics[width=0.45\hsize]{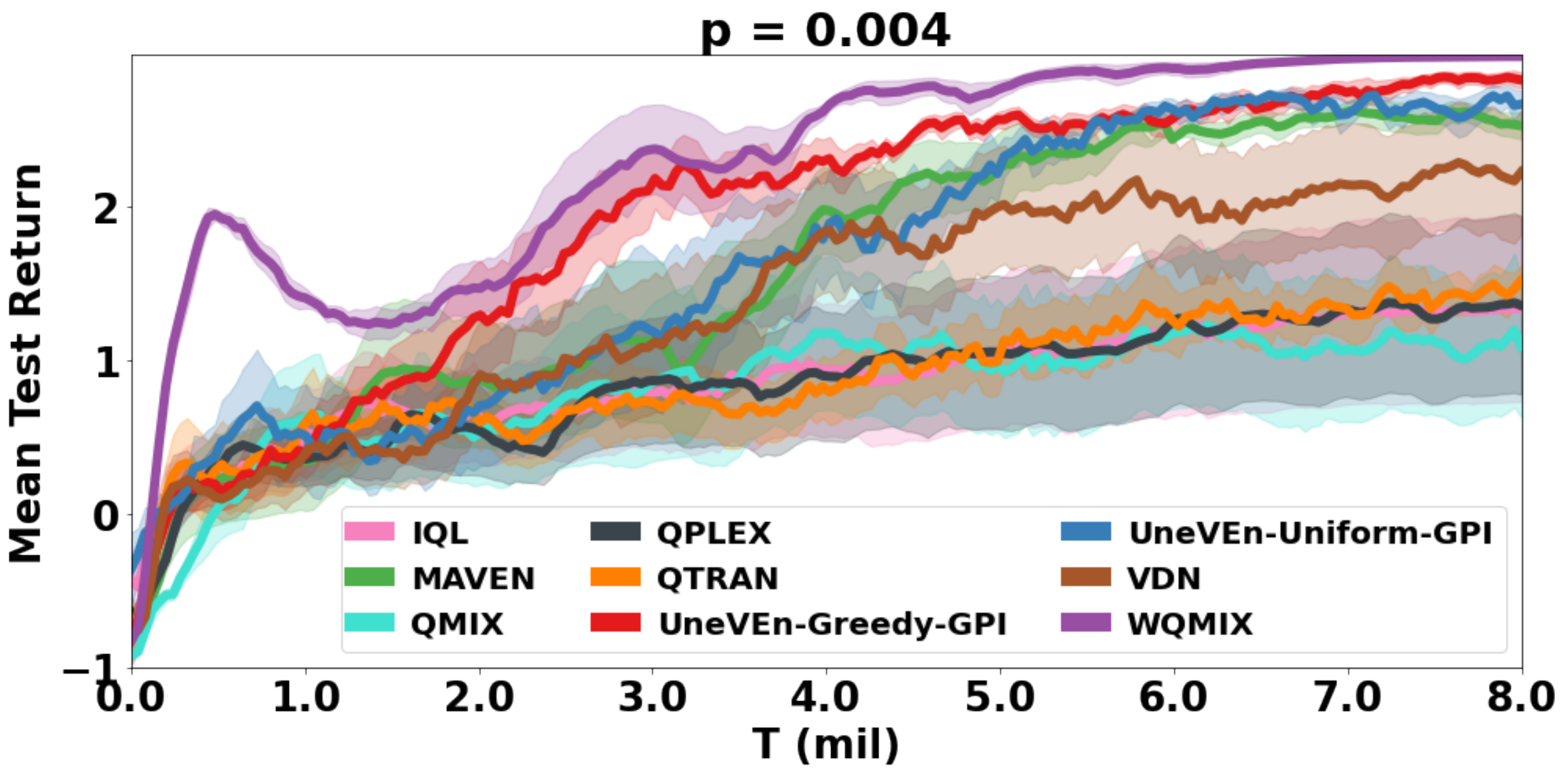}} 
\subfigure{\includegraphics[width=0.45\hsize]{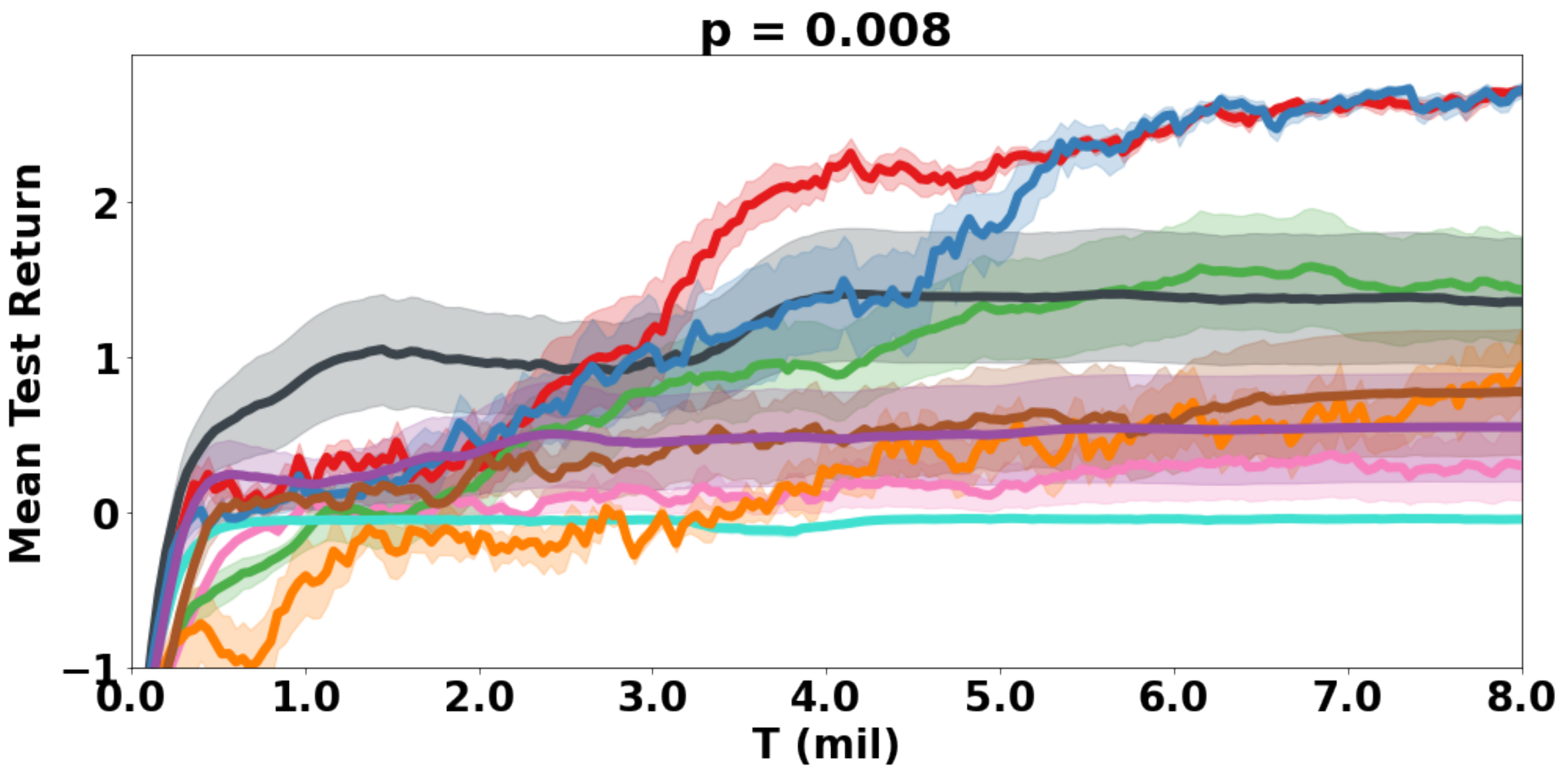}}
\par
\subfigure{\includegraphics[width=0.45\hsize]{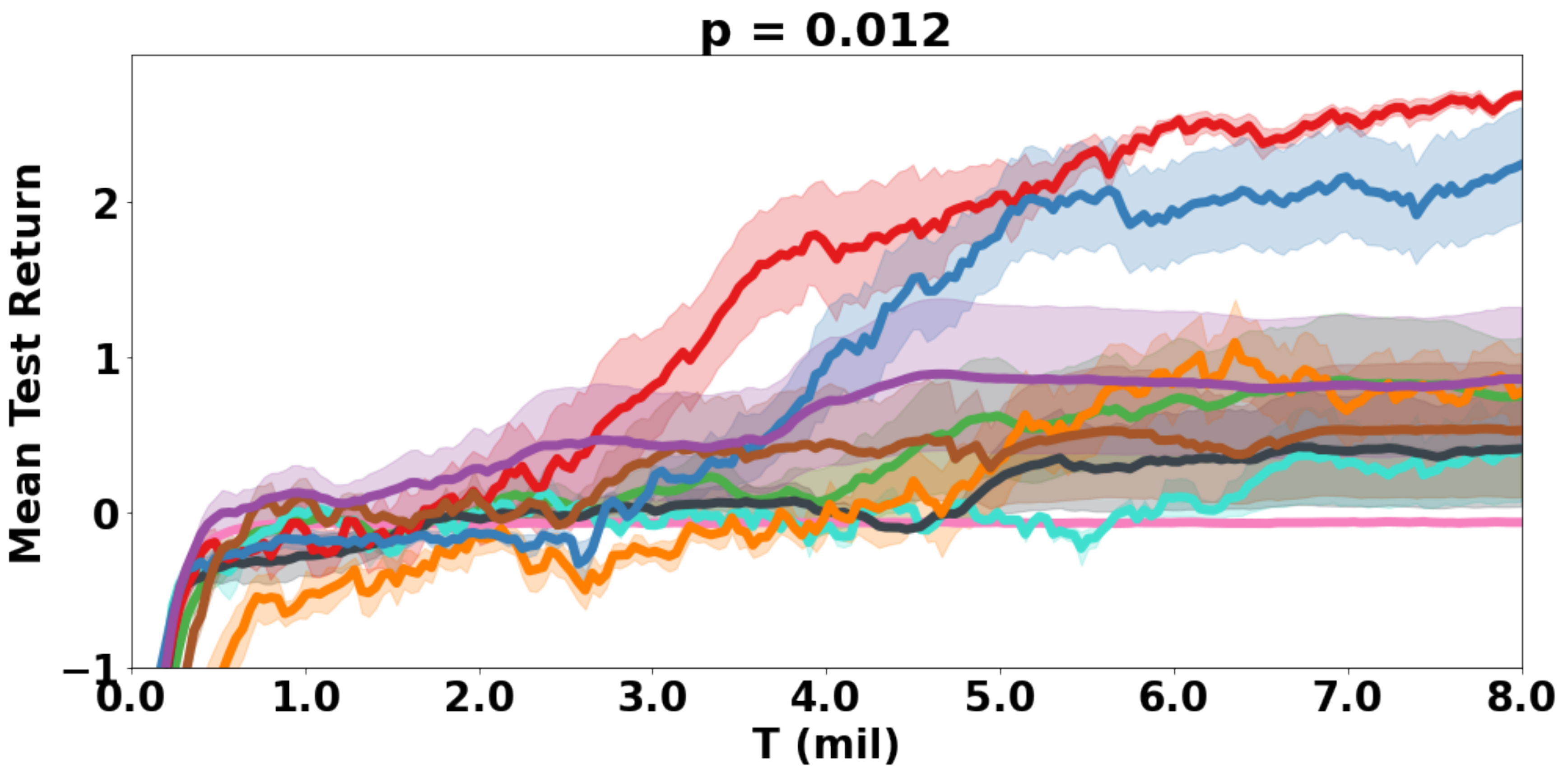}}
\subfigure{\includegraphics[width=0.45\hsize]{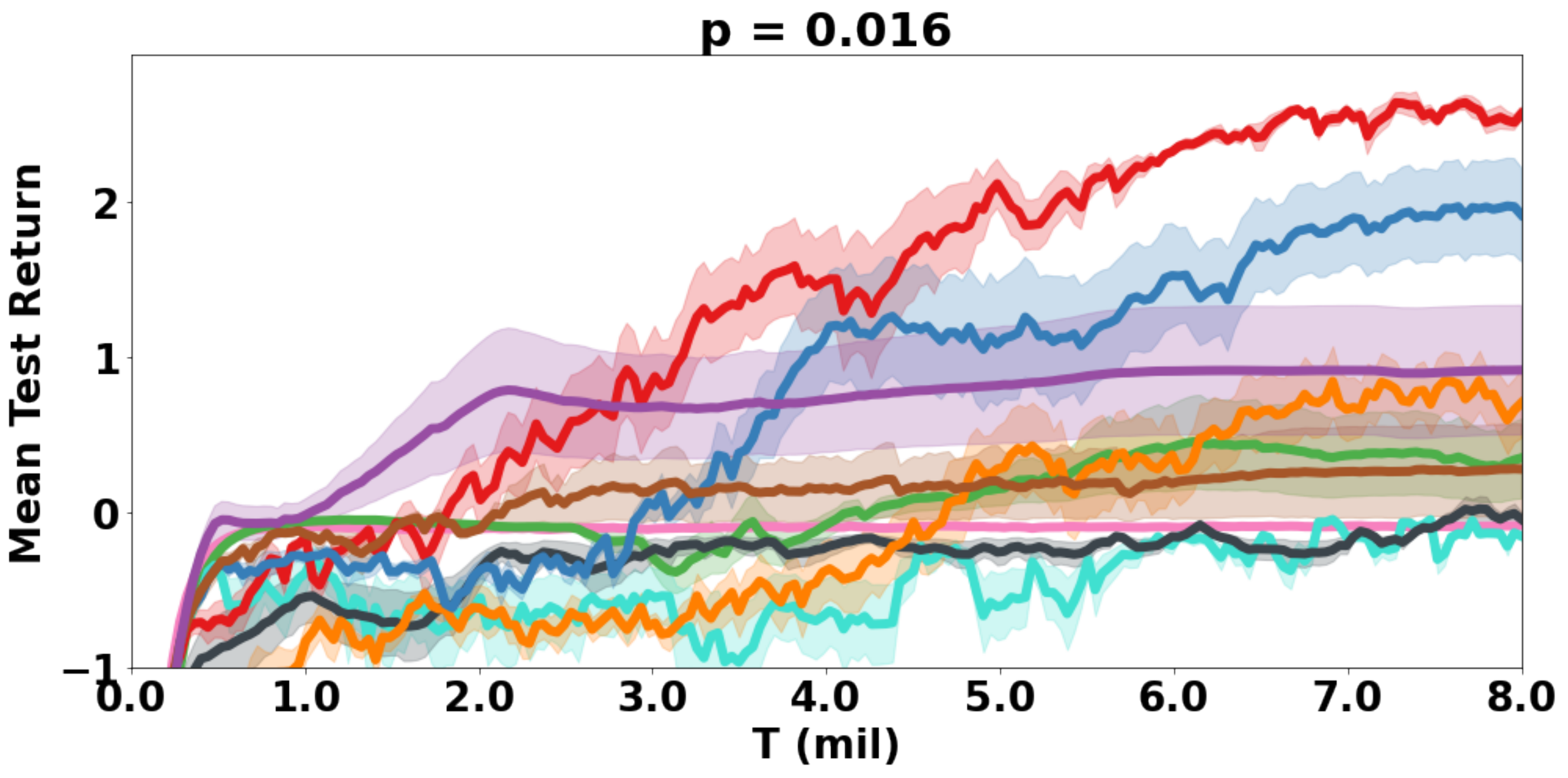}}
\caption{Comparison between \shorttitle and SOTA MARL baselines with $p \in \{0.004, 0.008, 0.012, 0.016\}$.}
\label{Fig3:all}
\end{figure*}

\begin{figure}[b!]
	\centering
	\includegraphics[width=.9\columnwidth]{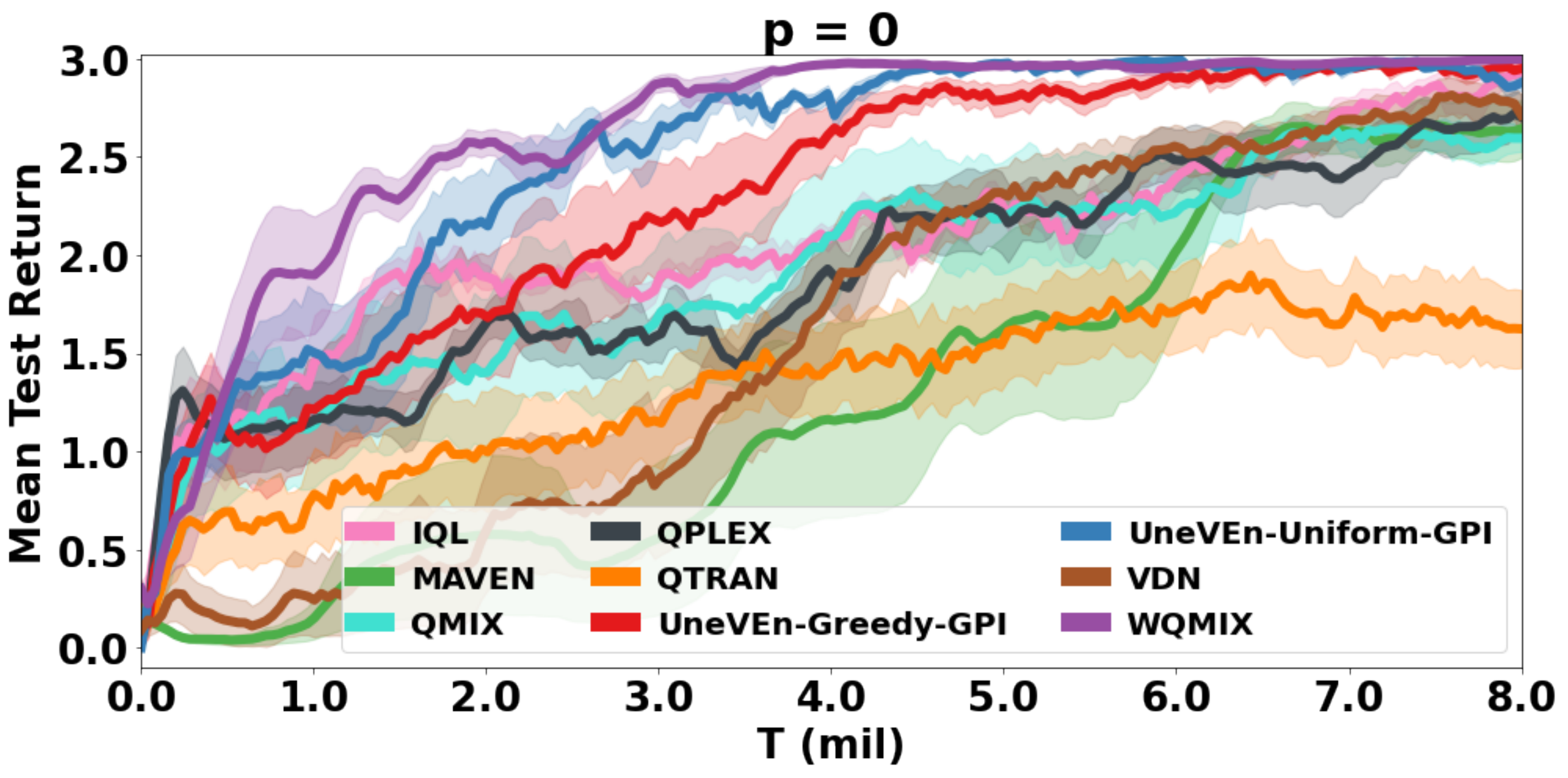}
	\caption{Baseline predator-prey results without RO ($p = 0$).}
	\label{fig:0.0}
\end{figure}

We first evaluate \shorttitle on the $m$-step matrix game proposed by \citet{mahajan2019maven}.
This task is difficult to solve using simple $\epsilon$-greedy exploration policies as committed exploration is required to achieve the optimal return. Appendix \ref{app-b} shows the $m$-step matrix game in which the first joint decision of two agents determines the maximal outcome after another $m-1$ decisions. One initial joint action can reach a return of up to $m+3$, whereas another only allows for $m$. This challenges monotonic value functions, as the target task exhibits RO. 
Figure \ref{mstep:all} shows the results of all methods on this task for $m=10$ after training for $35k$ steps. \shorttitle with greedy (\shorttitle-Greedy-GPI) action selection scheme converges to an optimal return and both greedy and uniform (\shorttitle-Uniform-GPI) schemes outperforms all other methods, which suffer from poor $\epsilon$-greedy exploration and often learn to take the suboptimal action in the beginning. Due to 
RO, it becomes difficult to change the policy later, leading to suboptimal returns and only rarely converging to optimal solutions.

\textbf{Domain 2 : Cooperative Predator-Prey} 

We next evaluate \shorttitle on challenging cooperative predator-prey  tasks similar to one proposed by \citet{son2019qtran}, but significantly more complex in terms of the coordination required amongst agents. We use a complex partially observable predator-prey (PP)  task involving eight agents (predators) and three prey that is designed to test coordination between agents, as each prey needs to be captured by at least three surrounding agents with a simultaneous \textit{capture} action. If only one or two surrounding agents attempt to capture the prey, a negative reward of magnitude $p$ is given. Successful capture yields a positive reward of +1. 
More details about the task are available in Appendix \ref{app-b}.

%

\textbf{Simpler PP Tasks}: We first demonstrate that both VDN and QMIX with monotonic value functions can learn on target tasks with simpler reward functions. 
To generate a simpler task, we remove the penalty associated with miscoordination, i.e., $p=0$, thereby making the target task free from RO. 
Figure \ref{fig:0.0} shows that both QMIX and VDN can solve this task as there is no miscoordination penalty and the monotonic value function can learn to efficiently represent the optimal joint action values. 
Other SOTA value-based approaches (MAVEN, WQMIX, and QPLEX) and \shorttitle with both uniform (\shorttitle-Uniform-GPI) and greedy (\shorttitle-Greedy-GPI) action-selection schemes can solve this target task as well. 


\begin{figure*}[t!]
\centering     
\subfigure{\includegraphics[width=0.33\hsize]{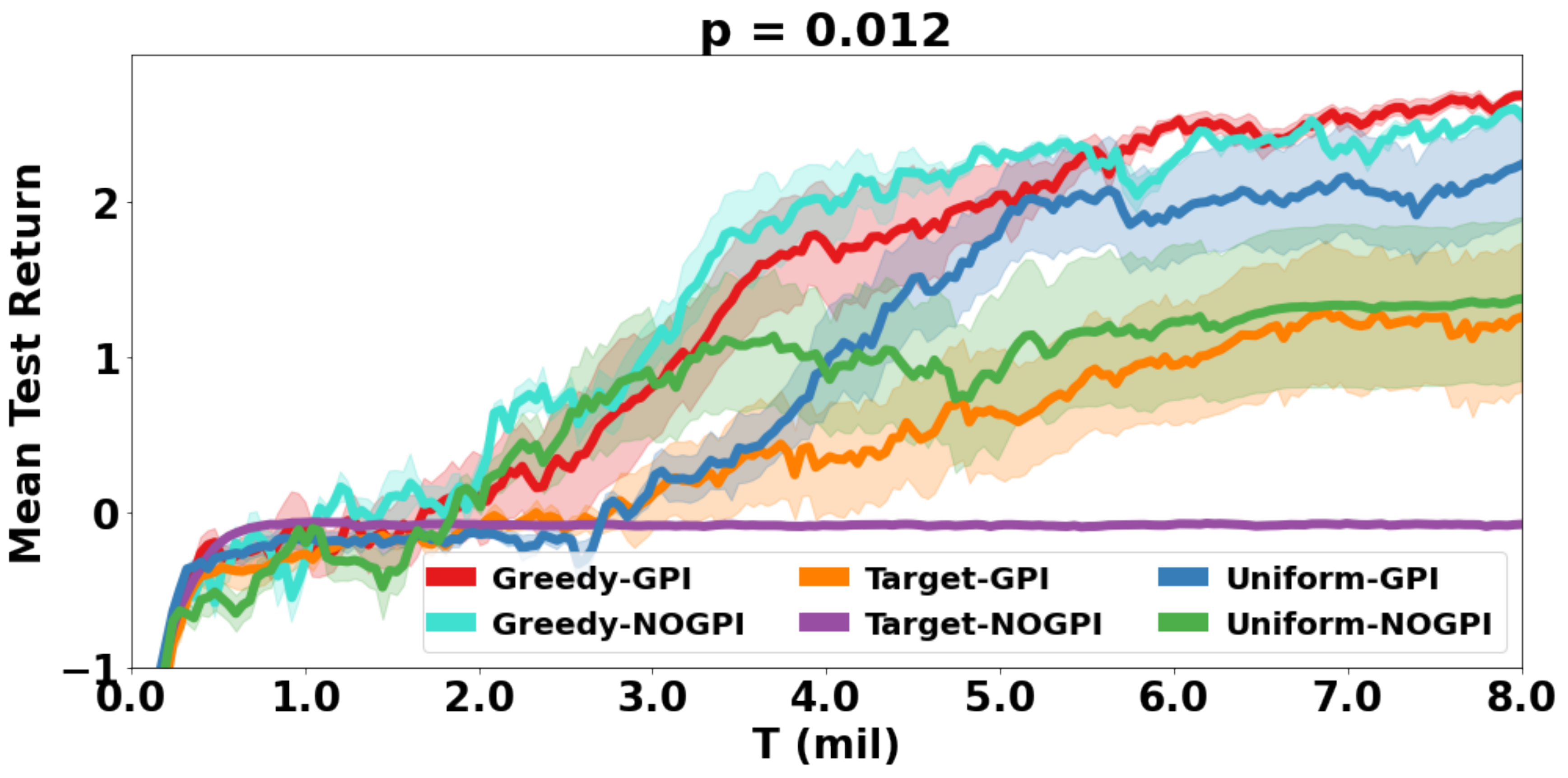}}
\subfigure{\includegraphics[width=0.33\hsize]{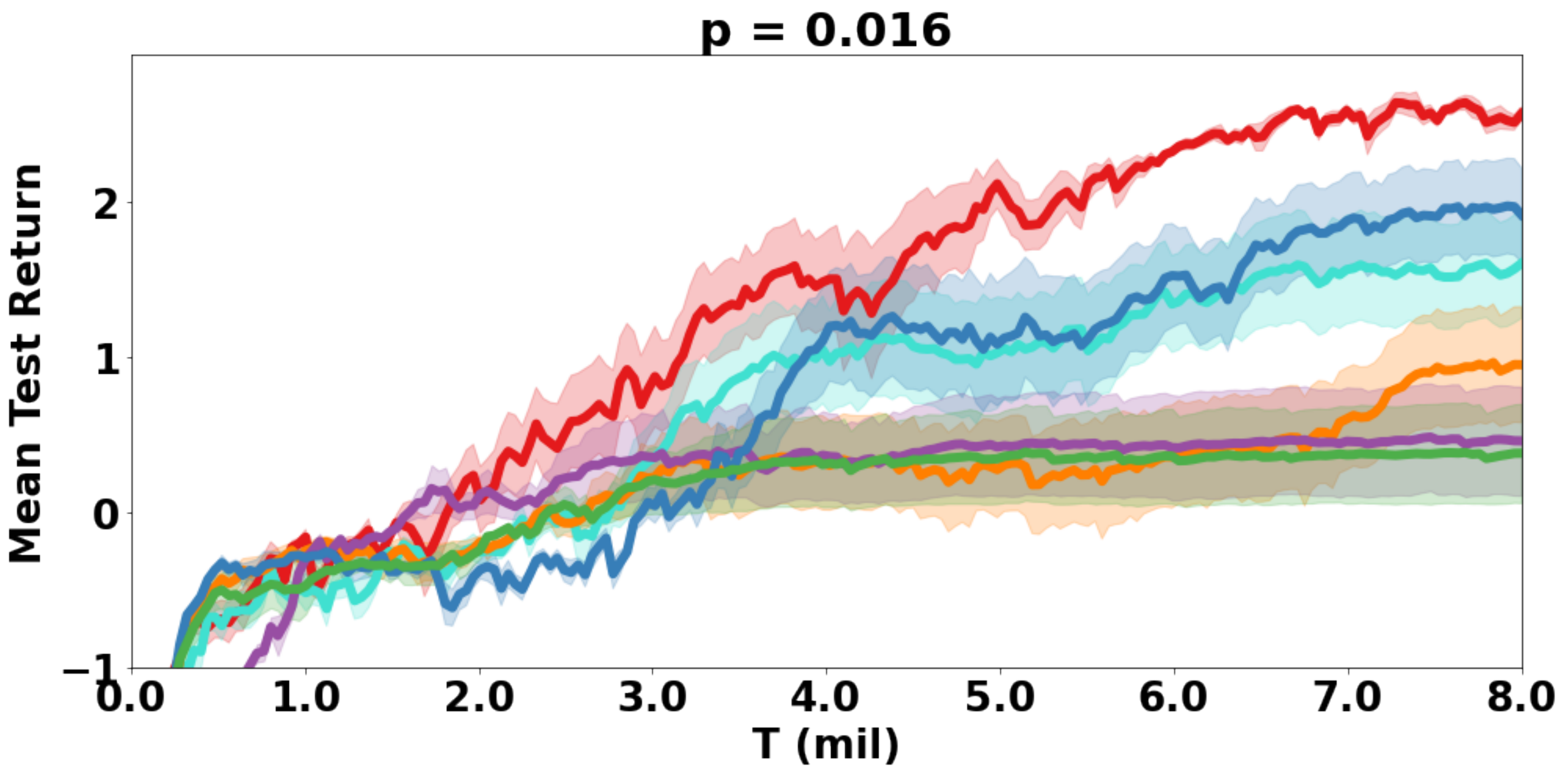}}
\subfigure{\includegraphics[width=0.33\hsize]{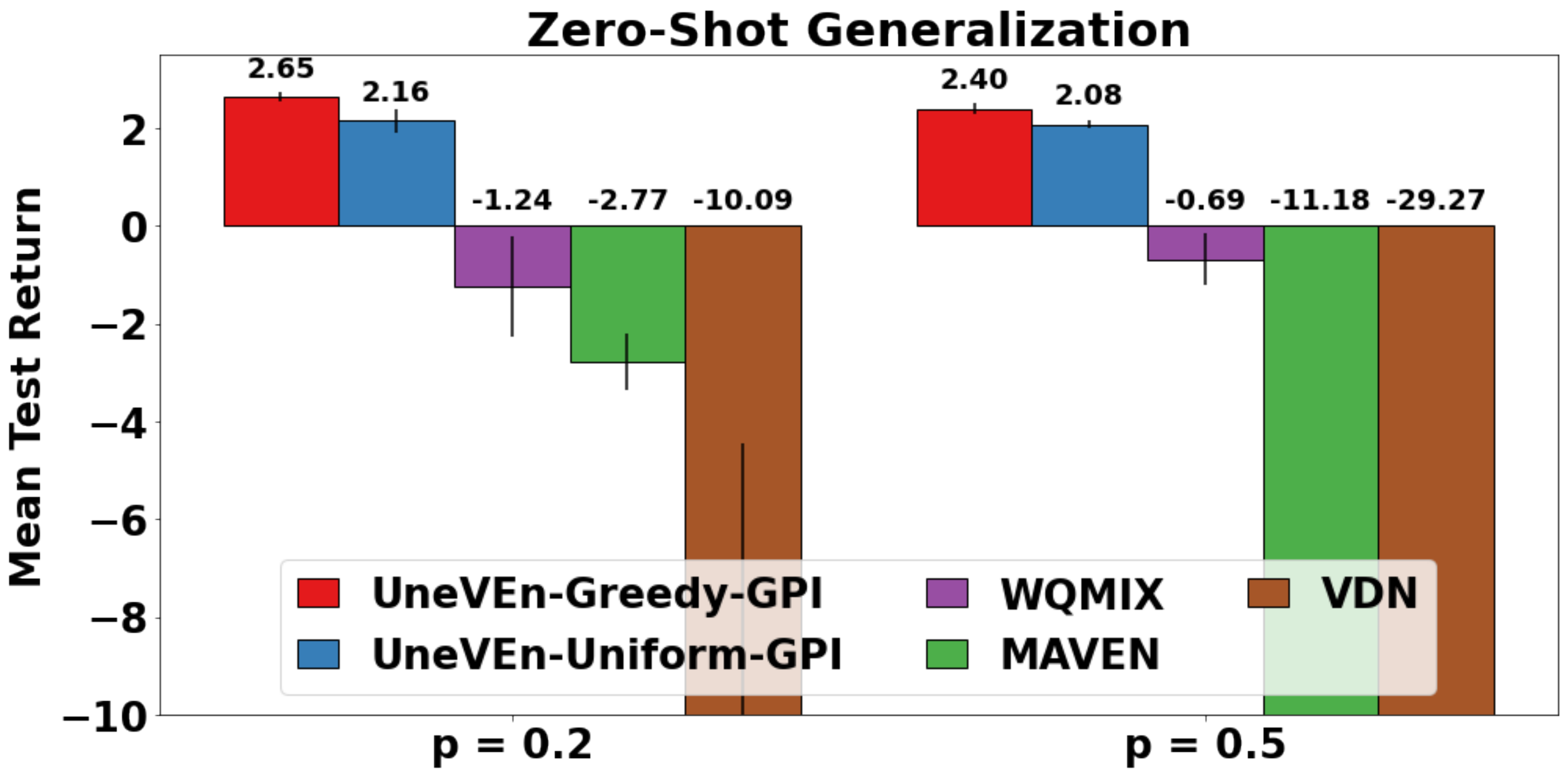}}
\caption{Ablation results. (left, middle): Comparison between different action selection schemes of \shorttitle for $p \in \{0.012, 0.016\}$. (right): Zero-shot generalization comparison; training on $p=0.004$, testing on $p \in \{0.2, 0.5\}$.
}
\label{Fig4:all}
\end{figure*}

\textbf{Harder PP Tasks}: We now raise the chance that RO occurs in the target task by increasing the magnitude of the penalty associated with each miscoordination, 
i.e., $p \in \{0.004, 0.008, 0.012, 0.016\}$. For a smaller penalty of $p = 0.004$, Figure \ref{Fig3:all} (top left) shows that VDN can still solve the task, further suggesting that simpler reward related tasks (with lower penalties) can be solved with monotonic approximations. However, both QMIX and VDN fail to learn on three other higher penalty target tasks due to their monotonic constraints, which hinder the accurate learning of the joint action value functions. 
Intuitively, when uncoordinated joint actions are much more likely than coordinated ones, the penalty term can dominate the average value estimated by each agent's utility.
This makes it difficult to learn an accurate monotonic approximation that  selects the optimal joint actions.  

Interestingly, other SOTA value-based approaches that aim to address the monotonicity restriction of QMIX and VDN such as MAVEN, QTRAN, WQMIX, and QPLEX also fail to learn on higher penalty tasks. WQMIX solves the task when $p = 0.004$, but fails on the other three higher penalty target tasks. Although WQMIX uses an explicit weighting mechanism to bias learning towards the target task's optimal joint actions, it must identify these actions by learning an unrestricted value function first. 
An $\epsilon$-greedy exploration based on the target task takes a long time to learn such a value function, which is visible in the large standard error for $p \in \{0.008, 0.012, 0.016\}$ in Figure \ref{Fig3:all}.
By contrast, both \shorttitle-Uniform-GPI and \shorttitle-Greedy-GPI can approximate  value functions of target tasks exhibiting severe RO more accurately and solve the task for all values of $p$. As expected, the variance of \shorttitle-Uniform-GPI is high on higher penalty target tasks (for e.g., $p = 0.016$) as exploration suffers from action selection based on harder related tasks. \shorttitle-Greedy-GPI does not suffer from this problem. Videos of learnt policies are available at \url{https://sites.google.com/view/uneven-marl/}.

 
We emphasize that while having an unrestricted joint value function (such as WQMIX, QPLEX and QTRAN) may allow to overcome RO, these algorithms are not guaranteed to do so in any reasonable amount of time.
For example, \citet{rashid2020weighted} note in Sec.~6.2.2 that training such a joint function can require much longer $\epsilon$-greedy exploration schedules.
The exact optimization of QTRAN is computationally intractable and a lot of recent work has shown the unstable performance of the corresponding approximate versions. Similarly QPLEX has shown to perform poorly on hard SMAC maps in our results and in \citet{rashid2020weighted}.
On the other hand, MAVEN learns a diverse ensemble of monotonic action-value functions, but does not concentrate on joint actions that would overcome RO if explored more.

\textbf{Ablations}: 
Figure \ref{Fig4:all} shows ablation results for higher penalty tasks, i.e., $p = \{0.012, 0.016\}$. To contrast the effect of \shorttitle on exploration, we compare our two novel action-selection schemes to \shorttitle-Target-GPI, which only selects the greedy actions of the target task.
The results clearly show that \shorttitle-Target-GPI  fails to solve the higher penalty RO tasks as the employed monotonic joint value function of the target task fails to accurately represent the values of different joint actions. This demonstrates the critical role of \shorttitle and its action-selection schemes.  

Next we evaluate the effect of GPI by comparing against \shorttitle with MAUSFs without using the GPI policy, i.e., setting the evaluation set $\mathcal{C}_1 = \{\bs{k}\}$ in \Eqref{scheme-gpi}. First, \shorttitle using a NOGPI policy with both uniform (Uniform-NOGPI) and greedy (Greedy-NOGPI) action selection outperforms Target-NOGPI, further strengthening the claim that \shorttitle with its novel action-selection scheme enables efficient exploration and bias towards optimal joint actions. In addition, the left and middle plots of Figure \ref{Fig4:all} shows that for each corresponding action-selection scheme (uniform, greedy, and target), using a GPI policy ($*$-GPI) is consistently favourable as it performs either similarly to the NOGPI policy ($*$-NOGPI)  or much better. GPI appears to improve zero-shot generalization of MAUSFs across tasks, which in turn enables good action selection for related tasks during \shorttitle exploration.


\begin{figure*}[htbp]
\centering     
\subfigure{\includegraphics[width=0.33\hsize]{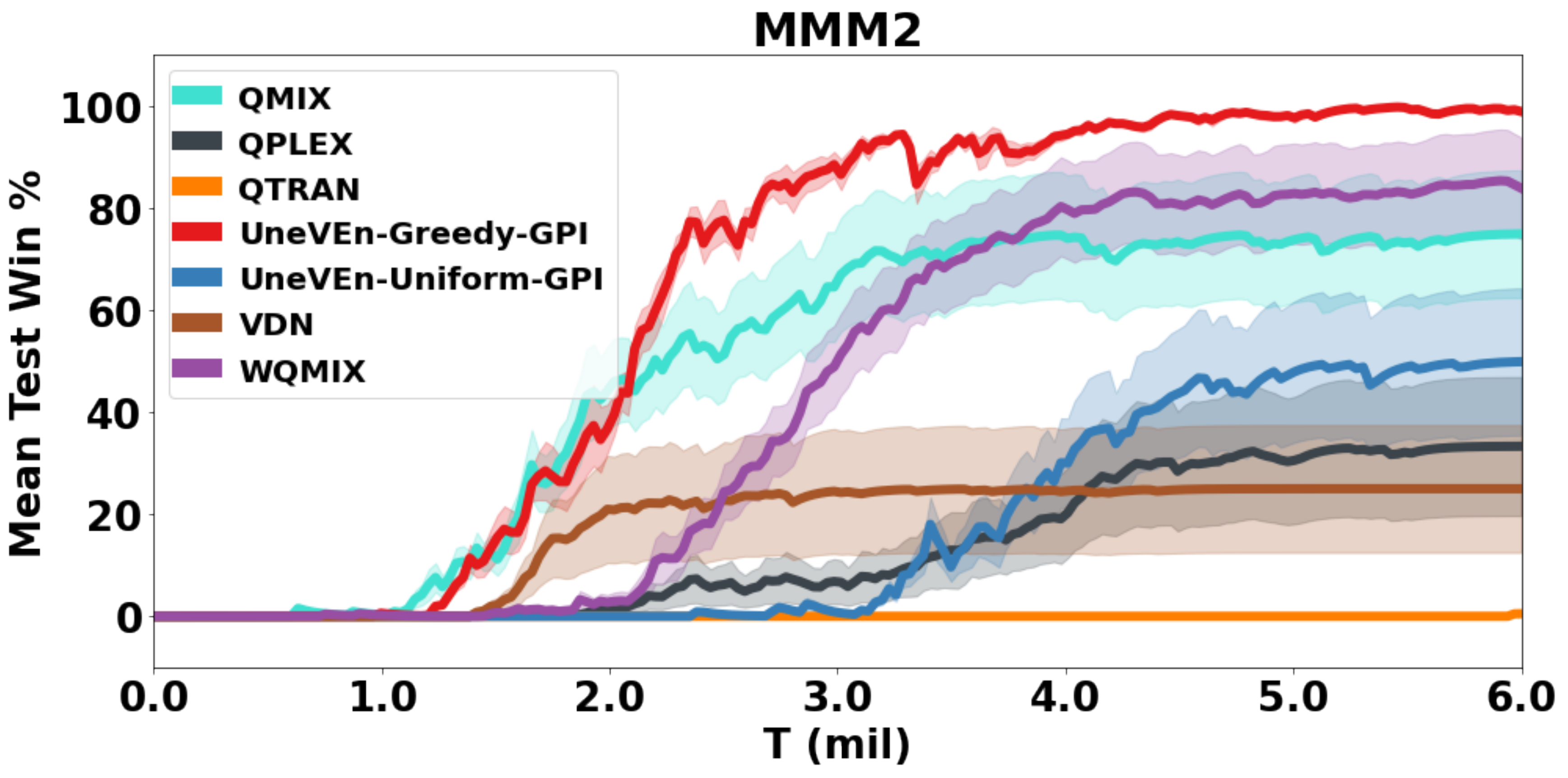}} 
\subfigure{\includegraphics[width=0.33\hsize]{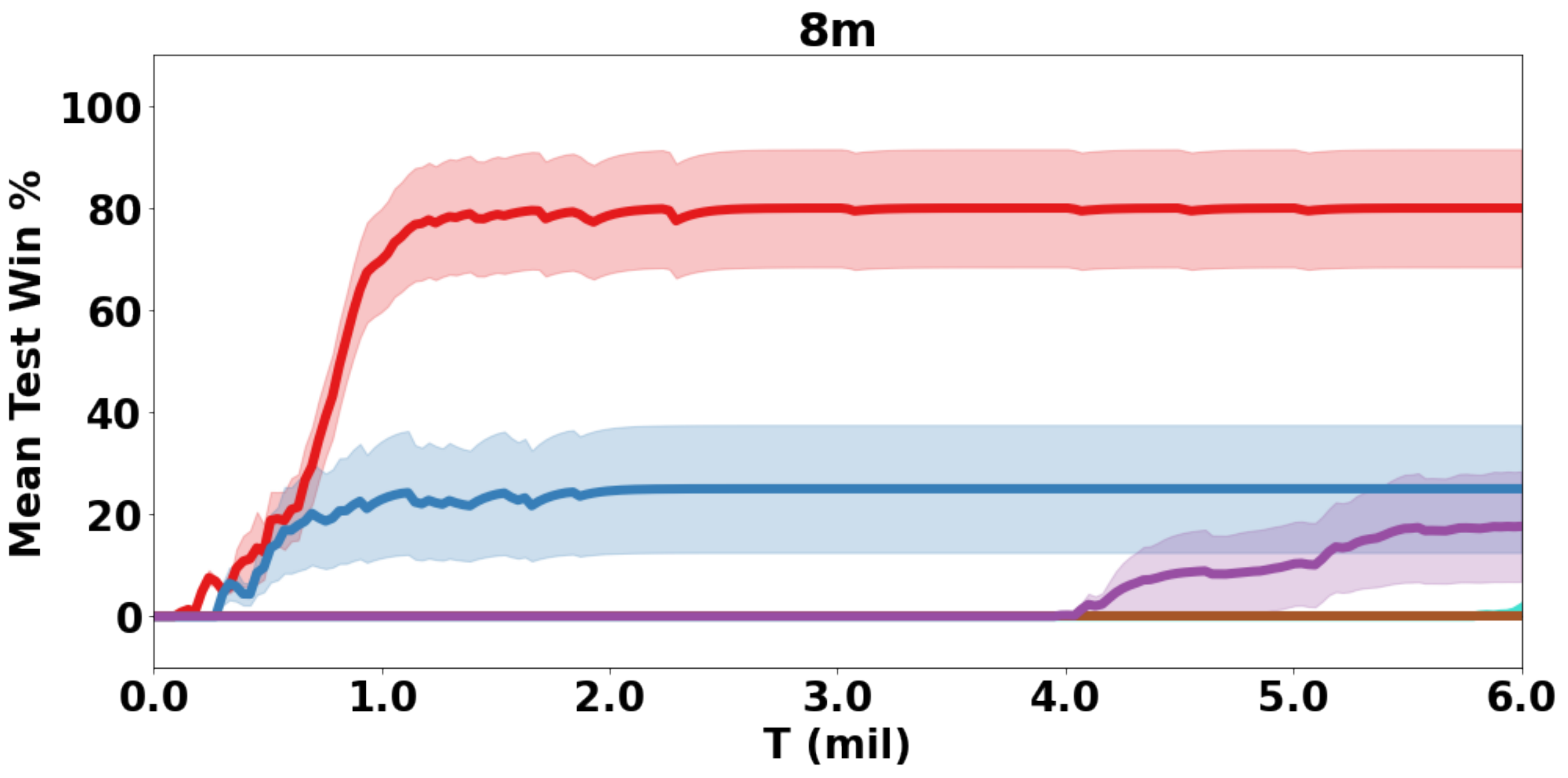}}
\subfigure{\includegraphics[width=0.33\hsize]{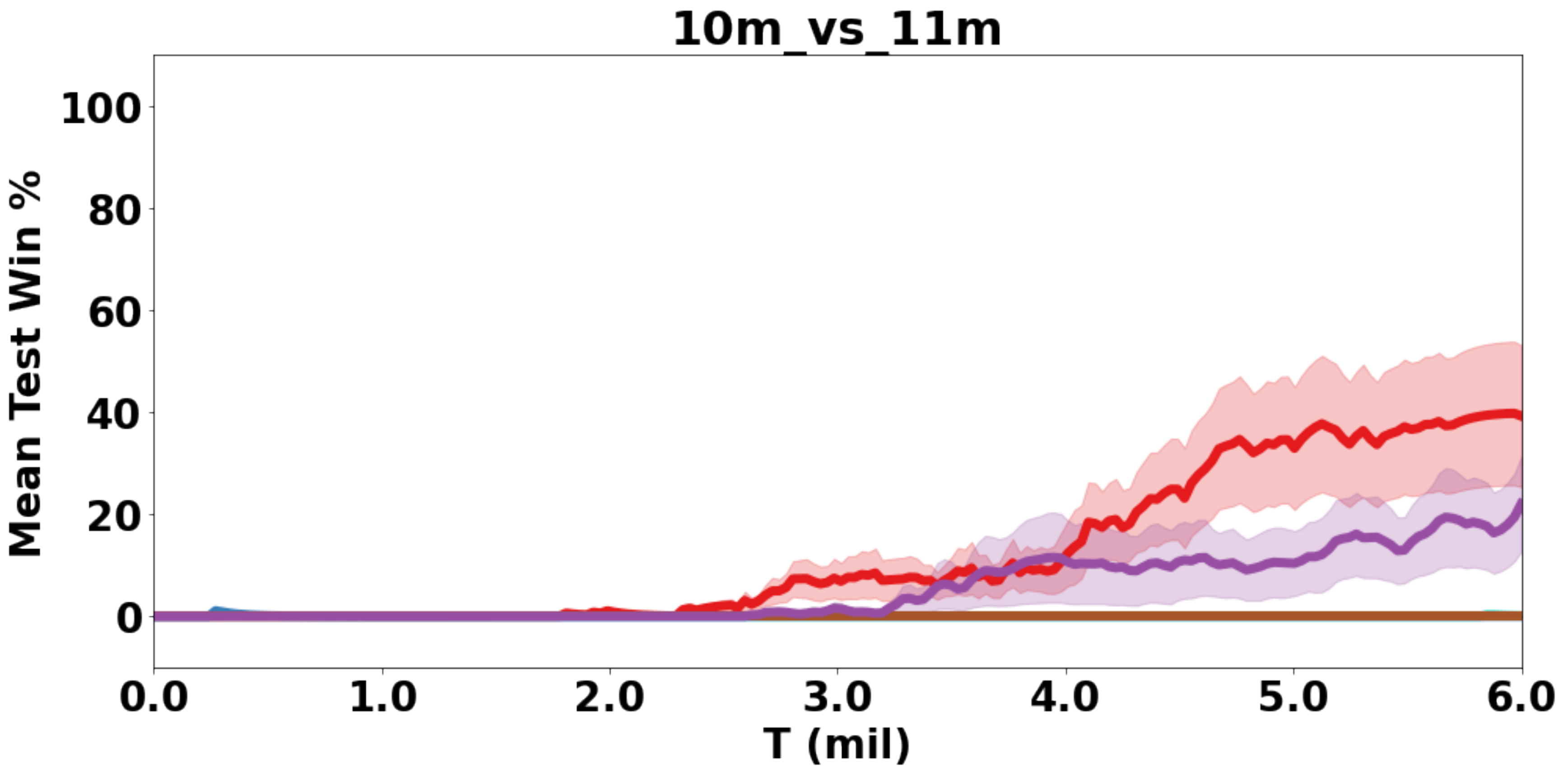}}
\caption{Comparison between \shorttitle and SOTA MARL baselines on SMAC maps with $p = 1.0$.}
\label{smac:all}
\end{figure*}

\textbf{Zero-Shot Generalization}: Lastly, we evaluate this zero-shot generalization for all methods to check if the learnt policies are useful for unseen high penalty test tasks. %
We train all methods for $8$ million environmental steps on a task with $p=0.004$, and test 60 rollouts of the resulting policies of all methods that are able to solve the training task, i.e., \shorttitle-Greedy-GPI, \shorttitle-Uniform-GPI, VDN,  MAVEN, and WQMIX, on tasks with $p\in \{0.2, 0.5\}.$
For policies trained with \shorttitle-Greedy-GPI and \shorttitle-Uniform-GPI, we use the NOGPI policy for the zero-shot testing, i.e.,
$\mathcal{C}_1 = \mathcal{C}_2 = \{\bs{w}\}$. 
The right plot of Figure \ref{Fig4:all} shows that \shorttitle with both uniform and greedy schemes exhibits great zero-shot generalization and solves both test tasks even with very high penalties. As MAUSFs learn the reward's basis functions, rather than the reward itself, zero-shot generalization to larger penalties follows naturally. Furthermore, using \shorttitle exploration allows the agents to collect enough diverse behaviour to come up with a near optimal policy for the test tasks.
On the other hand, the learnt policies for all other methods that solve the target task with $p = 0.004$ are ineffective in these higher penalty RO tasks, as they do not learn to avoid unsuccessful capture attempts.
See Appendix \ref{app-c} for details about the implementations and Appendix \ref{app-d} for additional ablation experiments.


\textbf{Domain 3 : Starcraft Multi-Agent Challenge (SMAC)}

We now evaluate \shorttitle on challenging cooperative StarCraft II (SC2) maps from the popular SMAC benchmark \citep{samvelyan2019starcraft}. 
Our method aims to overcome relative overgeneralization (RO) which happens very often in practice with cooperative games \citep{wei2016lenient}. However, the default reward function in SMAC does not suffer from RO as it has been designed with QMIX in mind, deliberately dropping any punishment for loosing ally units and thereby making it solvable for all considered baselines.
We therefore consider SMAC maps where each ally agent unit is additionally penalized ($p$) for being killed or suffering damage from the enemy, in addition to receiving positive reward for killing/damage on enemy units. A detailed discussion about different reward functions in SMAC along with their implications are discussed in Appendix \ref{app-d}. 
%
%
The additional penalty to the reward function (similar to QTRAN++, \citet{son2020qtran}) for losing our own ally units induces RO in the same way as in our predator prey example. This makes maps that were classified as ``easy" by default benchmark (e.g.,~``\texttt{8m}'' in Fig. 8) very hard to solve for methods like VDN and QMIX for $p=1.0$ (equally weighting the lives of ally and enemy units). 

Prior work on SC2 has established that VDN/QMIX can solve nearly all SMAC maps in the absence of any punishment (i.e., negative reward scaling $p = 0$) and we confirm this also holds for low punishments (see results in Figure \ref{Fig13:all} for $p = 0.5$ in Appendix \ref{app-d}). However, this puts little weight on the lives of your own allies and cautious generals may want to learn a more conservative policy with higher punishment for losing their ally units. Similarly to the predator-prey example, Figure \ref{smac:all} shows that increasing the punishment to $p=1$ (equally weighting the lives of allies and enemies) leads VDN/QMIX and other SOTA MARL methods like QTRAN, WQMIX and QPLEX to fail in many maps, whereas our novel method UneVEn, in particular with greedy action selection, reliably outperforms baselines on all tested SMAC maps. 

\vspace{-1mm}
\section{Related Work}

\textbf{Improving monotonic value function factorization in CTDE:} MAVEN \citep{mahajan2019maven} shows that the monotonic joint action value function of QMIX and VDN suffers from suboptimal approximations on nonmonotonic tasks. 
It addresses this problem by learning a diverse ensemble of monotonic joint action-value functions conditioned on a latent variable by optimizing the mutual information between the joint trajectory and the latent variable. Deep Coordination Graphs (DCG) \citep{bohmer2019deep} uses a predefined coordination graph \citep{guestrin2002coordinated} to 
represent the joint action value function. However, DCG is not a fully decentralized approach and 
specifying the coordination graph can require significant domain knowledge. \citet{son2019qtran} propose QTRAN, which addresses the monotonic restriction of QMIX by learning a (decentralizable) VDN-factored joint action value function along with an unrestricted centralized critic. The corresponding utility functions are distilled from the critic by solving a linear optimization problem involving all joint actions, but its exact implementation is  computationally intractable and the corresponding approximate versions have unstable performance. \textsc{Tesseract} \citep{mahajan2021tesseract} proposes a tensor decomposition based method for learning arbitrary action-value functions, they also provide sample complexity analysis for their method. QPLEX \citep{wang2020qplex} uses a duplex dueling \citep{wang2016dueling} network architecture to factorize the joint action value function with linear decomposition structure. WQMIX \citep{rashid2020weighted} learns a QMIX-factored joint action value function along with an unrestricted centralized critic 
and proposes explicit weighting mechanisms to bias the monotonic approximation of the optimal joint action value function towards important joint actions, which is similar to our work. However, in our work, the weightings are implicitly done through action-selection based on related tasks, which are easier to solve.

\textbf{Exploration:} 
There are many techniques for 
exploration in model-free single-agent RL, based on  
intrinsic novelty reward \citep{bellemare2016count, tang2017hash},
predictability \citep{pathak2017exploration}, 
pure curiosity \citep{burda2019curiosity}
or Bayesian posteriors \citep{osband2016rvf, gal2017dropout, fortunato18noisynet, odonoghue2018ube}.
In the context of multi-agent RL,
\citet{boehmer19unreliable} discuss the influence of unreliable intrinsic reward
and \citet{wang2020influence} quantify the influence 
that agents have on each other's return. Similarly, \citet{wang2020rode} propose a learnable action effect based role decomposition which eases exploration in the joint action space.
\citet{zheng2018structured} propose to coordinate 
exploration between agents by shared latent variables, whereas
\citet{jaques2018social} investigate the social motivations of competitive agents.
However, these techniques aim to visit as much 
of the state-action space as possible,
which exacerbates the relative overgeneralization pathology.
Approaches that use state-action space abstraction can speed up exploration, these include those which can automatically learn the abstractions \citep[e.g.,][]{mahajan2017symmetryde, mahajan2017symmetryl} and those which require prior knowledge \citep[e.g.,][]{roderick18abstract}, however they are difficult to scale for multi-agent scenarios.
In contrast, \shorttitle explores similar {\em tasks}.
This guides exploration to states and actions that prove {\em useful},
which restricts the explored space and overcomes relative overgeneralization.  
To the best of our knowledge, 
the only other work that explores the task space 
is \citet{leibo19autocurricula}:
they use the evolution of competing agents
as an auto-curriculum of harder and harder tasks.
Collaborative agents cannot compete against each other, though, 
and their approach therefore is not suitable to learn cooperation.


\textbf{Successor Features:} Most of the work on SFs have been focused on single-agent settings \citep{dayan1993improving, kulkarni2016deep, lehnert2017advantages, barreto2017successor, barreto2018transfer, borsa2018universal, lee2019truly, hansen2019fast}  for transfer learning and zero-shot generalization across tasks with different reward functions. 
\citet{gupta2019successor} use single-agent SFs in a multi-agent setting to estimate the probability of events, but they only consider transition independent MARL \citep{becker2004solving, Gupta2018PlanningAL}. 

\section{Conclusion}
This paper presents a novel approach decomposing multi-agent universal successor features (MAUSFs) into local agent-specific SFs, which enable a decentralized version of the GPI to maximize over a combinatorial space of agent policies. We propose \shorttitle, which leverages the generalization power of MAUSFs to perform action selection based on simpler related tasks to address the suboptimality of the target task's monotonic joint action value function in current SOTA methods. Our experiments show that \shorttitle significantly outperforms VDN, QMIX and other state-of-the-art value-based MARL methods on challenging tasks exhibiting severe RO. 

\section{Future Work}
\shorttitle relies on the assumption that both the basis functions and the reward-weights of the target task are known which restricts the applicability of the method. Moreover, our SFs based method requires learning in $d$ dimensions instead of learning scalar values for each action in the case of $Q$-learning, which decreases the scalability of our method for domains where $d$ is very large. More efficient neural network architectures with hyper networks \citep{ha2016hypernetworks} can be leveraged to handle larger dimensional features. Finally, the paradigm of UneVEn with related task based action selection can be directly applied with universal value functions (UVFs) \citep{schaul2015universal} to enable other state-of-the-art nonlinear mixing techniques like QMIX \citep{rashid2020monotonic} and QPLEX \citep{wang2020qplex} to overcome RO, at the cost of losing ability to perform GPI.

\section*{Acknowledgements}
We thank Christian Schroeder de Witt, Tabish Rashid and other WhiRL members for their feedback. Tarun Gupta is supported by the Oxford University Clarendon Scholarship. Anuj Mahajan is funded by the J.P. Morgan A.I. fellowship. This project has received funding from the European Research Council under the European Union’s Horizon 2020 research
and innovation programme (grant agreement number 637713). The experiments were made possible by a generous equipment grant from NVIDIA.

\bibliography{references}
\bibliographystyle{icml2021}

\onecolumn
\appendix
\section{Formal derivation of the illustrative example in Section \ref{sec:example}}
\label{sec:app-example}
A VDN decomposition of the joint state-action-value function 
in the simplified predator-prey task shown in Figure \ref{fig:spp_env} is:
$$
	Q^\pi(s, u^1\!, u^2) \quad=\quad 
	\E\Big[ r(s,u^1\!,u^2) + \gamma V^\pi(s') 
		\,\Big|\!\!\begin{array}{c}
			\scriptstyle u^2 \sim \pi^2(\cdot|s) \\[-1mm]
			\scriptstyle s' \sim P(\cdot|s,u^1\!\!,u^2)
		\end{array}\!\!\Big]
	\quad\approx\quad 
	Q^1(s,u^1) + Q^2(s, u^2) 
	\quad =: \quad Q_\text{tot}(s, u^1, u^2) \,.
$$
Here $V^\pi(s')$ denotes the expected return of state $s' \in \Set S$. 
However, in this analysis we are mainly interested in the behavior at 
the beginning of training, and so we assume in the following that 
the estimated future value is close to zero, 
i.e.~$V^\pi(s') \,\approx\, 0, \forall s' \in \Set S$.
The VDN decomposition allows to derive a decentralized policy
$\pi(u^1,u^2|s) := \pi^1(u^1|s) \, \pi^2(u^2|s)$,
where $\pi^i$ is usually an $\epsilon$-greedy policy
based on agent $i$'s utility $Q^i(s, u^i)$, $i \in \{1,2\}$.
Similar to \citet{rashid2020weighted},
we analyze here the corresponding VDN projection operator:
$$
	\Pi_\text{vdn}^\pi [Q^\pi] \quad:=\quad \argmin_{Q_\text{tot}}
	\sum_{u^1 \in \Set U^1} \sum_{u^2 \in \Set U^2}
		\pi(u^1, u^2|s) \, \Big( 
			Q_\text{tot}(s, u^1, u^2) - Q^\pi(s, u^1, u^2)
		\Big)^{\!2} \,.
$$
Setting the gradient of the above mean-squared loss 
\wbr{}{(with $V^\pi(s')=0$) }to zero yields
for $Q^1$ (and similarly for $Q^2$):
$$
	Q^1(s, u^1) \quad \stackrel{!}{=} \quad \sum_{u^2 \in \Set U^2}
		\pi^2(u^2|s) \Big( \underbrace{r(s, u^1, u^2) 
		+ \gamma \E[\, V^\pi(s')\,|\, {\scriptstyle s' \sim P(\cdot|s,u^1,u^2)}]
		}_{Q^\pi(s, u^1, u^2)} - Q^2(s,u^2) \Big) \,.
$$ 
In the following we use $\bar V_{u^1} := \max_{u^2} 
\E[\, V^\pi(s')\,|\, {\scriptstyle s' \sim P(\cdot|s,u^1,u^2)}]$.
Relative overgeneralization \citep{panait2006biasing}
occurs when an agent's utility (for example agent $1$'s)
of the catch action $u^1=C$
falls below the movement actions $u^1 \in \{L,R\}$: 
$$
\begin{array}{rrcl}
	& Q^1(s, C) & < & Q^1(s, L/R) 
\\[2mm] \Leftrightarrow & \qquad
	\smallsum{u^2 \in \Set U^2}{} \pi^2(u^2|s) \, 
		\big( r(s, C, u^2) - Q^2(s,u^2) \big)
	& < &
	\smallsum{u^2 \in \Set U^2}{} \pi^2(u^2|s) \, 
		\big( Q^\pi(s, L/R, u^2) - Q^2(s,u^2) \big)
\\[2mm] \Rightarrow &
	r \, \pi^2(C|s) - p \, \pi^2(L|s) - p \, \pi^2(R|s) 
	& < & 
	-p \, \pi^2(C|s) \, + \,\gamma \big(1 - \pi^2(C|s)\big) \bar V_{L/R}
\\[2mm] \Leftrightarrow &
	r \, \pi^2(C|s) 
	& < & 
	p \, \big(\pi^2(L|s) + \pi^2(R|s) - \pi^2(C|s) \big)
	\, + \,\gamma \big(1 - \pi^2(C|s)\big) \bar V_{L/R}
\\[2mm] \Leftrightarrow &
	r \, \pi^2(C|s) 
	& < & 
	p \, \big(1 - 2 \pi^2(C|s) \big)
	\, + \,\gamma \big(1 - \pi^2(C|s)\big) \bar V_{L/R}
\\[2mm] \Leftrightarrow &
	r  
	& < & 
	p \, \big(\smallfrac{1}{\pi^2(C|s)} - 2 \big)
	\, + \,\gamma \big(1 - \pi^2(C|s)\big) \bar V_{L/R}
\end{array}
$$
This demonstrates that, at the beginning of training with $\bar V_{L/R}=0$,
relative overgeneralization occurs at $p > r$ for $\pi^2(C|s)=\frac{1}{3}$,
at $p > \frac{3}{2}\, r$ for $\pi^2(C|s)=\frac{3}{8}$,
at $p > 3 \, r$ for $\pi^2(C|s)=\frac{3}{7}$,
and at {\em never} at $\pi^2(C|s) > \frac{1}{2}$,
as shown in Figure \ref{fig:spp_tasks}.
Note that the value $\bar V_{L/R}$ contributes a constant
w.r.t.~punishment $p$, that is scaled by $1 - \pi^2(C|s)$,
and that positive values make the problem harder (require lower $p$).
An initial relative overgeneralization will make agents choose the wrong
greedy actions, which lowers $\pi^2(C|s)$ even further and therefore 
solidifies the pathology when $\epsilon$ get's annealed.
Empirical thresholds in experiments can differ significantly, 
though, as random sampling increases the chance 
to over/underestimate $\pi^2(C|s)$ during the exploration phase
and the estimated value $V^\pi(s')$ has a non-trivial influence.
Furthermore, the above thresholds cannot be transferred directly 
to the experiments conducted in Section \ref{sec:experiments},
which are much more challenging:
(i) it requires 3 agents to catch a prey, 
(ii) agents have to explore 5 actions,
(iii) agents can execute catch actions when they are alone with the prey,
and (iv) catch actions do not end the episode,
increasing the influence of empirically estimated $V^\pi(s')$.
Empirically we observe that VDN fails to solve tasks
somewhere between $p = 0.004 \, r$ and $p=0.008\,r$.
The presented analysis provides therefore only
an illustrative example how \shorttitle can help 
to overcome relative overgeneralization.

\section{Training algorithm} \label{app-a}
Algorithm \ref{train-mausf} and \ref{train-uneven} presents the training of MAUSFs with \shorttitle. Our method is able to learn on all tasks (target $\bs{w}$ and sampled $\bs{z}$) simultaneously in a sample efficient manner using the same feature $\bs\phi_t \equiv \bs{\phi}(s_t, \bs{u}_t)$ due to the linearly decomposed reward function (\Eqref{uvf-q}). 

\begin{algorithm}[htbp]
\caption{Training MAUSFs with UneVEn}
\label{train-mausf}
\begin{algorithmic}[1]
\STATE \textbf{Require} $\epsilon$, $\alpha$, $\beta$ target task $\bs{w}$, set of agents $\Set A$, standard deviation $\sigma$
\vspace{1mm}
\STATE Initialize the local-agent SF network $\bs{\psi}^a(\tau^a, u^a, \boldsymbol{z}; \theta)$ and replay buffer $\Set M$
\vspace{1mm}
\FOR {fixed number of \textbf{epochs}}
\STATE $\nu \sim \mathcal{N}(\bs{w}, \sigma \textbf{I}_d)$\,; \quad
		$\bs o_0 \leftarrow$ \textsc{ResetEnv}()
\STATE $t = 0 \,; \quad \Set M \leftarrow$ \textsc{NewEpisode}$(\Set M, \nu, \bs o_0)$
\vspace{1mm}
\WHILE {not \textbf{terminated}}
\STATE \textbf{if} Bernoulli($\epsilon$)=1 then $\bs{u}_t \leftarrow \text{Uniform}(\bs{\mathcal{U}})$
\STATE \textbf{else}  $\bs{u}_t \wbr{}{= \{u_t^a\}} \leftarrow \textsc{UneVEn}(\bs{\tau}_{t}, \nu)$
\STATE $\langle \bs o_{t+1}, \bs\phi_t \rangle \leftarrow$ \textsc{EnvStep}$(\bs u_t)$\STATE $\Set M \leftarrow$ \textsc{AddTransition}$\big(\Set M, \bs u_t, \bs o_{t+1}, \bs \phi_t\big)$
\STATE $t \leftarrow t + 1$
\ENDWHILE
\vspace{1mm}
\STATE $\Set L \leftarrow 0\,;\quad \mathcal{B} \leftarrow$ \textsc{SampleMiniBatch}$(\Set M)$
\vspace{1mm}
\FORALL{$\{\bs\tau_t, \bs u_t, \bs\phi_t, \bs\tau_{t+1}, \nu\} \in \Set B$}
\vspace{1mm}
\FORALL{$\bs{z} \in \nu \cup \{\bs{w}\}$} 
\STATE $\bs{u}_{\bs{z}}' \leftarrow \big\{\argmax\limits_{u \in \mathcal{U}} \bs{\psi}^a(\tau^a_{t+1}, u, \bs{z}; \theta)^\top \bs{z} \big\}_{a \in \Set A}$ 
\vspace{-0.5mm}
\STATE $\Set L \leftarrow \Set L + \big\|
\bs{\phi}_t + \gamma \bs{\psi}_{tot}(\bs\tau_{t+1}, \bs{u'}_{\bs{z}}, \bs z; \theta^{-}) - \bs{\psi}_{tot}(\bs\tau_t, \bs u_t, \bs z; \theta) \big\|^2_2$
\ENDFOR
\ENDFOR
\STATE $\theta \leftarrow$ \textsc{Optimize}$(\theta, \nabla_{\!\theta} \Set L)$
\STATE $\theta^- \leftarrow (1-\beta)\,\theta^- + \beta \, \theta$
\ENDFOR

\end{algorithmic}
\end{algorithm}

\begin{algorithm}[htbp]
\caption{$\textsc{UneVEn}(\bs{\tau}_{t}, \nu)$}
\label{train-uneven}
\begin{algorithmic}[1]
\IF{Bernoulli$(\alpha)=1$ \textbf{or} Scheme is {\bf Target}}
\STATE $\Set C_2 \leftarrow \{\bs w\}$
\ELSE
\IF{Scheme is \textbf{Uniform}}
\STATE $\mathcal{C}_2 \leftarrow \nu \sim \text{Uniform}(\nu)$
\ELSIF{Scheme is \textbf{Greedy}}
\STATE $\mathcal{C}_2 \leftarrow \nu \cup \{\bs{w}\}$
\ENDIF
\ENDIF
\IF{Use\_GPI\_Policy is \textbf{True}}
\STATE $\mathcal{C}_1 \leftarrow   \nu \cup \{\bs{w}\}$
\STATE $\bs{u}_{t} \leftarrow  \{ u_{t}^a = \argmax\limits_{u \in \mathcal{U}}\max\limits_{\bs{k} \in \mathcal{C}_2} \max\limits_{\bs{z} \in \mathcal{C}_1} \bs{\psi}^a(\tau^a_t, u, \bs{z}; \theta)^\top \bs{k}\}_{a \in \mathcal{A}}$
\vspace{-2mm}
\ELSE
\STATE $\bs{u}_{t} \leftarrow  \{ u_{t}^a = \argmax\limits_{u \in \mathcal{U}}\max\limits_{\bs{k} \in \mathcal{C}_2} \bs{\psi}^a(\tau^a_t, u, \bs{k}; \theta)^\top \bs{k}\}_{a \in \mathcal{A}}$
\vspace{-2mm}
\ENDIF
\STATE \textbf{return} $\bs{u}_{t}$
\end{algorithmic}
\end{algorithm}

\section{Experimental Domain Details and Analysis} \label{app-b}
\subsection{Predator-Prey}
 We consider a complicated partially observable predator-prey (PP)  task in an $10 \times 10$ grid involving eight agents (predators) and three prey that is designed to test coordination between agents, as each prey needs a simultaneous \textit{capture} action by at least three surrounding agents to be captured. Each agent can take 6 actions i.e. move in one of the 4 directions (Up, Left, Down, Right), remain still (no-op), or try to catch (capture) any adjacent prey. The prey moves around in the grid with a probability of $0.7$ and remains still at its position with probability $0.3$. Impossible actions for both agents and prey are marked unavailable, for eg. moving into an occupied cell or trying to take a capture action with no adjacent prey.

If either a single or a pair of agents take a capture action on an adjacent prey, a negative reward of magnitude $p$ is given. If three or more agents take the capture action on an adjacent prey, it leads to a successful capture of that prey and yield a positive reward of $+1$. The maximum possible reward for capturing all prey is therefore $+3$. Each agent observes a $5 \times 5$ grid centered around its position which contains information showing other agents and prey relative to its position. An episode ends if all prey have been captured or after $800$ time steps. This task is  similar to one proposed by \citet{bohmer2019deep, son2019qtran}, but significantly more complex in terms of the coordination required amongst agents as more agents need to coordinate simultaneously to capture the prey. 

This task is challenging for two reasons. First, depending on the magnitude of $p$, exploration is difficult as even if a single agent miscoordinates, the penalty is given, and therefore, any steps toward successful coordination are penalized. Second, the agents must be able to differentiate between the values of successful and unsuccessful collaborative actions, which monotonic value functions fail to do on tasks exhibiting RO.

\begin{table}[t!]
\centering 
\begin{tabular}{ cc } 
A1-Capture & A1-Other \\  
\begin{tabular}{ccc}
\hline
           & A2-Capture & A2-Other \\ \hline
A3-Capture & +1         & $-p$      \\
A3-Other   & $-p$         & $-p$      \\ \hline
\end{tabular} &  
\begin{tabular}{ccc}
\hline
           & A2-Capture & A2-Other \\ \hline
A3-Capture & $-p$        & $-p$      \\
A3-Other   & $-p$         & 0        \\ \hline
\end{tabular} \\
\end{tabular}
\caption{\label{joint-table} Joint-Reward function of three agents surrounding a prey. The first table indicates joint-rewards when Agent $1$ takes capture action and second table indicates joint-rewards when Agent $1$ takes any other action. Notice that there are numerous joint actions leading to penalty $p$. }
\end{table}

\begin{propositionNO}
For, the predator-prey game defined above, the optimal joint action reward function for any group of $2\leq k\leq n$ predator agents surrounding a prey is \textit{nonmonotonic}  \citep[as defined by][]{mahajan2019maven} iff $p > 0$.
\end{propositionNO}
\begin{proof}

Without loss of generality, we assume a single prey surrounded by three agents ($A_1, A_2, A_3$) in the environment. The joint reward function for this group of three agents is defined in Table \ref{joint-table}.

For the case $p>0$ the proposition can be easily verified using the definition of non-monotonicity~\citep{mahajan2019maven}. For any $3\leq k\leq n$ agents attempting to catch a prey in state $s$, we fix the actions of any $k-3$ agents to be ``\textit{other}" indicating either of up, down, left, right, and noop actions and represent it with $\bs{u}^{k-3}$. Next we consider the rewards $r$ for two cases: 
\begin{itemize}
\item  If we fix the action of any \textit{two} of the remaining three agents as ``other" represented as $\bs{u}^{2}$, the action of the remaining agent becomes $u^{1} = \argmax_{u \in \mathcal{U}} r(s, \left\langle u, \bs{u}^2,\bs{u}^{k-3} \right\rangle)=$ ``other".
\item If we fix the $\bs{u}^2$ to be ``capture", we have : $u_1 = \argmax_{u \in\mathcal{U}} r(s, \left\langle u, \bs{u}^2,\bs{u}^{k-3} \right\rangle)=$ ``capture".
\end{itemize}
Thus the best action for agent $A_1$ in state $s$ depends on the actions taken by the other agents and the rewards $R(s)$ are non-monotonic. Finally for the equivalence, we note that for the case $p=0$ we have that a default action of ``capture" is always optimal for any group of $k$ predators surrounding the prey. Thus the rewards are monotonic as the best action for any agent is independent of the rest. 
\end{proof}

\subsection{$m$-Step Matrix Games}
Figure \ref{mstep:game} shows the $m$-step matrix game for $m=10$ from \citet{mahajan2019maven}, where there are $m-2$ intermediate steps, and selecting a joint-action with zero reward leads to termination of the episode.

\begin{figure}[t!]
\centering     
\subfigure{\includegraphics[width=0.8\hsize]{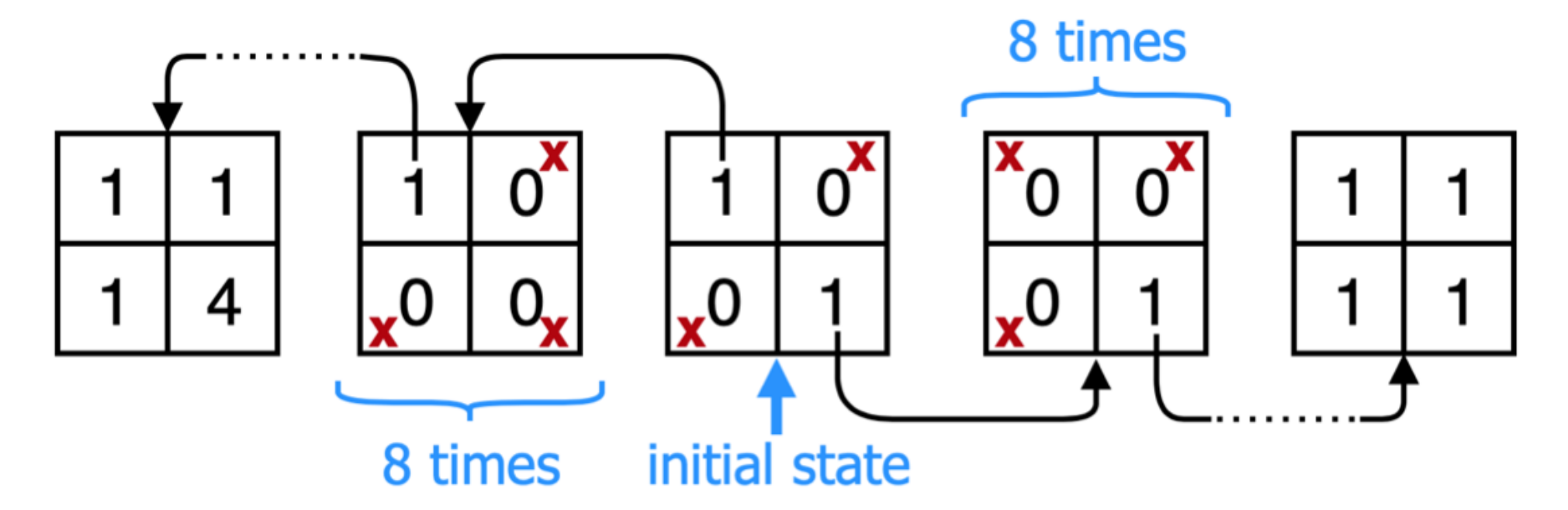}}
\caption{\small $m$-step matrix game from \citet{mahajan2019maven} for $m=10$. The red cross means that selecting that joint action will lead to termination of the episode.}
\label{mstep:game}
\vspace{-5mm}
\end{figure}

\subsection{StarCraft II Micromanagement}
We use the negative reward version of the SMAC benchmark \citep{samvelyan2019starcraft} where each ally agent unit is additionally penalized (penalty $p$) for being killed or suffering damage from the enemy, in addition to receiving positive reward for killing/damage on enemy units, which has recently been shown to improve performance \citep{son2020qtran}. We consider two versions of the penalty $p$, i.e. $p=0.5$ which is default in the SMAC benchmark and $p=1.0$ which equally weights the lives of allies and enemies, making the task more prone to exhibiting RO.

\section{Implementation Details} \label{app-c}
\subsection{Hyper parameters}

\begin{figure}[t!]
\centering     
\subfigure{\includegraphics[width=0.49\hsize]{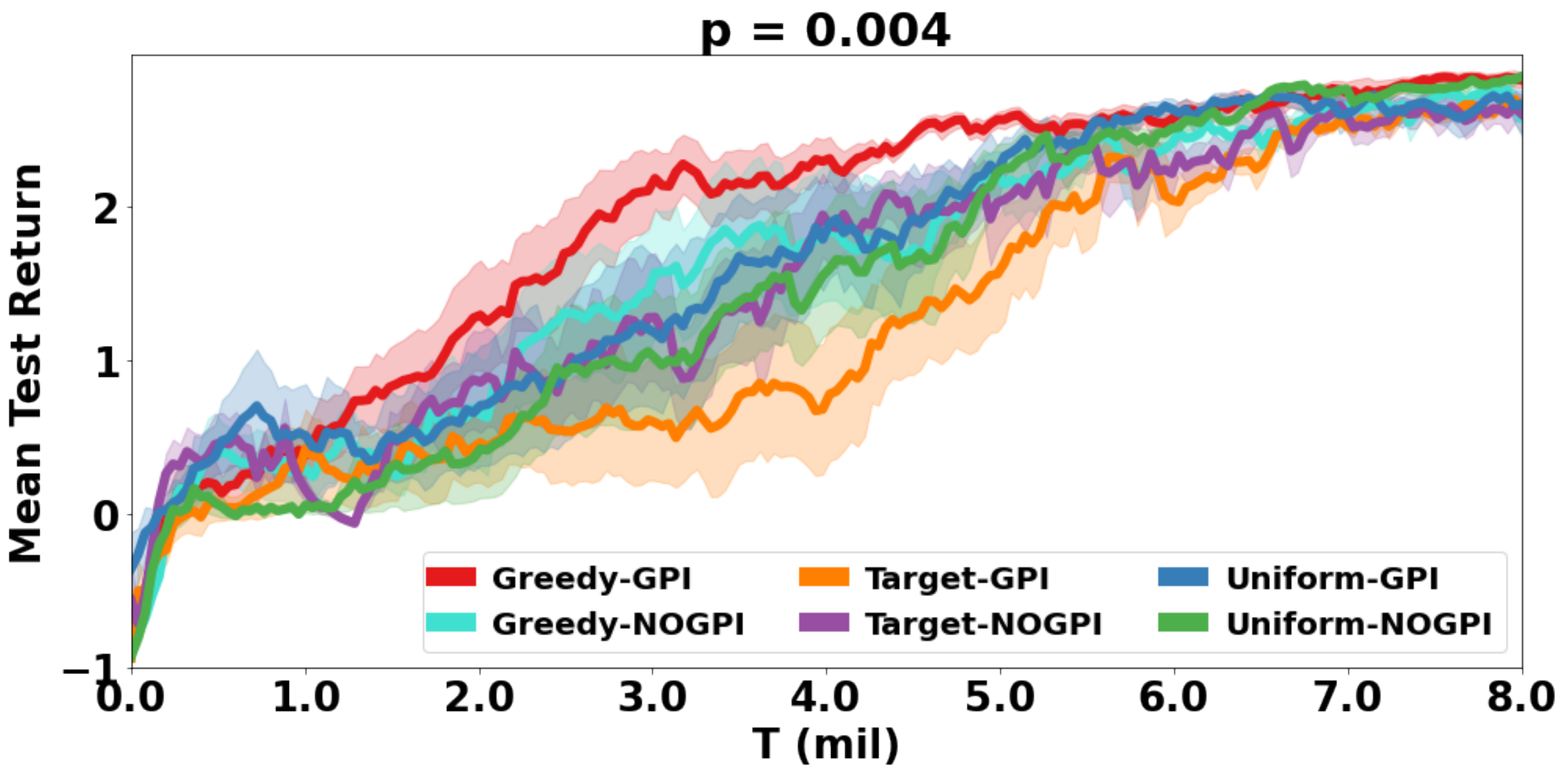}}
\subfigure{\includegraphics[width=0.49\hsize]{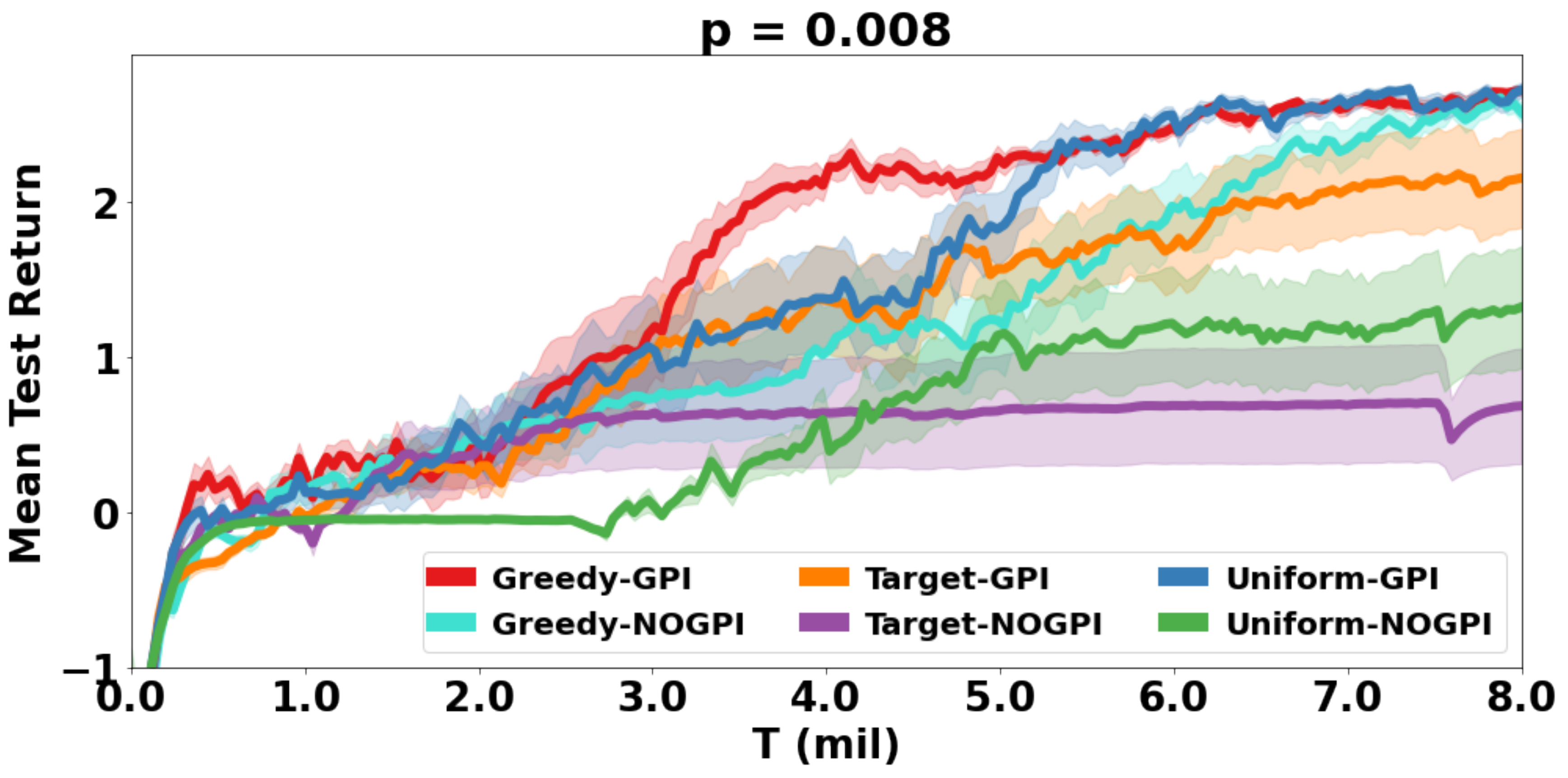}}
\caption{\small Additional Ablation results: Comparison between different action selection of \shorttitle for $p \in \{0.004, 0.008\}$.
}
\label{Fig6:all}
\end{figure}

All algorithms are implemented in the PyMARL framework \citep{samvelyan2019starcraft}. All our experiments use $\epsilon$-greedy scheme where $\epsilon$ is decayed from $\epsilon=1$ to $\epsilon=0.05$ over \{$250k$, $500k$\} time steps. All our tasks use a discount factor of $\gamma=0.99$. We freeze the trained policy every $30k$ timesteps and run $20$ evaluation episodes with $\epsilon=0$. We use learning rate of $0.0005$ with soft target updates for all experiments. We use a target network similar to \citet{mnih2015human} with ``soft" target updates, rather than directly copying the weights: $\theta^- \leftarrow \beta*\theta + (1-\beta)*\theta^-$, where $\theta$ are the current network parameters. We use $\beta=0.005$ for PP and $m$-step experiments and $\beta=0.05$ for SC2 experiments. This means that the target values are constrained to change slowly, greatly improving the stability of learning. All algorithms were trained with RMSprop optimizer by one gradient step on loss computed on a batch of 32 episodes sampled from a replay buffer containing last $1000$ episodes (for SC2, we use last $3000$ episodes). We also used gradient clipping to restrict the norm of the gradient to be $\leq 10$. 

The probability $\alpha$ of action selection based on target task in \shorttitle with uniform and greedy action selection schemes increases from $\alpha=0.3$ to $\alpha=1.0$ over $\{250k, 500k\}$ time steps. For sampling related tasks using normal distribution, we use $\mathcal{N}(\bs{w}, \sigma \textbf{I}_d)$ centered around target task $\bs{w}$ with $\sigma \in \{0.1, 0.2\}$. At the beginning of each episode, we sample \textit{six} related tasks, therefore $|\nu|=6$ (for SC2, we use $|\nu|=3$).

\subsection{NN Architecture}
Each agent's local observation $o_{t}^a$ are concatenated with agent's last action $u_{t-1}^a$, and then passed through a fully-connected (FC) layers of 128 neurons (for SC2, we use 1024 neurons), followed by ReLU activation, a GRU \citep{chung2014empirical}, and another FC of the same dimensionality to generate a action-observation history summary for the agent. Each agent's task vector $\bs{z} \in \nu \cup \{\bs{w}\}$ is passed through a FC layer of 128 neurons (for SC2, we use 1024 neurons) followed by ReLU activation to generate an internal task embedding. The history and task embedding are concatenated together and passed through two hidden FC-256 layers (for SC2, FC-2048 layer) and ReLU activations to generate the outputs for each action. For methods with non-linear mixing such as QMIX \citep{rashid2020monotonic}, WQMIX \citep{rashid2020weighted}, and MAVEN \citep{mahajan2019maven}, we adopt the same hypernetworks from the original paper and test with either a single or double hypernet layers of dim $64$ utilizing an ELU non-linearity. For all baseline methods, we use the code shared publicly by the corresponding authors on Github.

\section{Additional Results} \label{app-d}
\subsection{Predator-Prey}
Figure \ref{Fig6:all} presents additional ablation results for comparison between \shorttitle with different action selection schemes for $p \in \{0.004, 0.008\}$.  Figure \ref{Fig7:all} presents additional zero-shot generalization results for policies trained on target task with penalty $p=0.004$ tested on tasks with penalty $p \in \{0.2, 0.3, 0.5, 1.0\}$. 
For \shorttitle-Greedy-GPI, we can observe that the average number of miscoordinated capture attempts per episode actually drops with $p$ and converges around 1.2, i.e.,  for return $R_p$, average mistakes per episode is $\frac{3-R_p}{p} = \{1.75, 1.7, 1.2, 1.2\}$ for $p \in \{0.2, 0.3, 0.5, 1.0\}$.

\subsection{StarCraft II Micromanagement}

\begin{figure}[t!]
\centering     
\subfigure{\includegraphics[width=0.99\hsize]{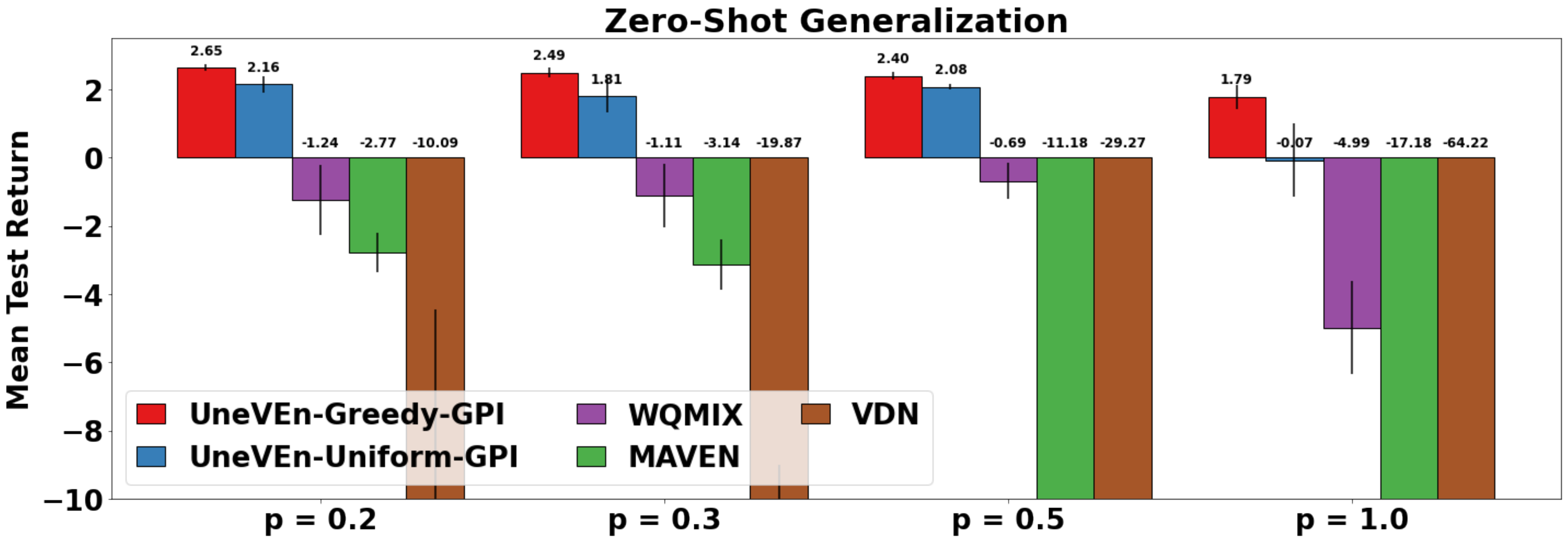}}
\caption{\small Additional Zero-shot generalization results for $p \in \{0.2, 0.3, 0.5, 1.0\}$.
}
\label{Fig7:all}
\end{figure}

\begin{figure}[t!]
\centering     
\subfigure{\includegraphics[width=0.32\hsize]{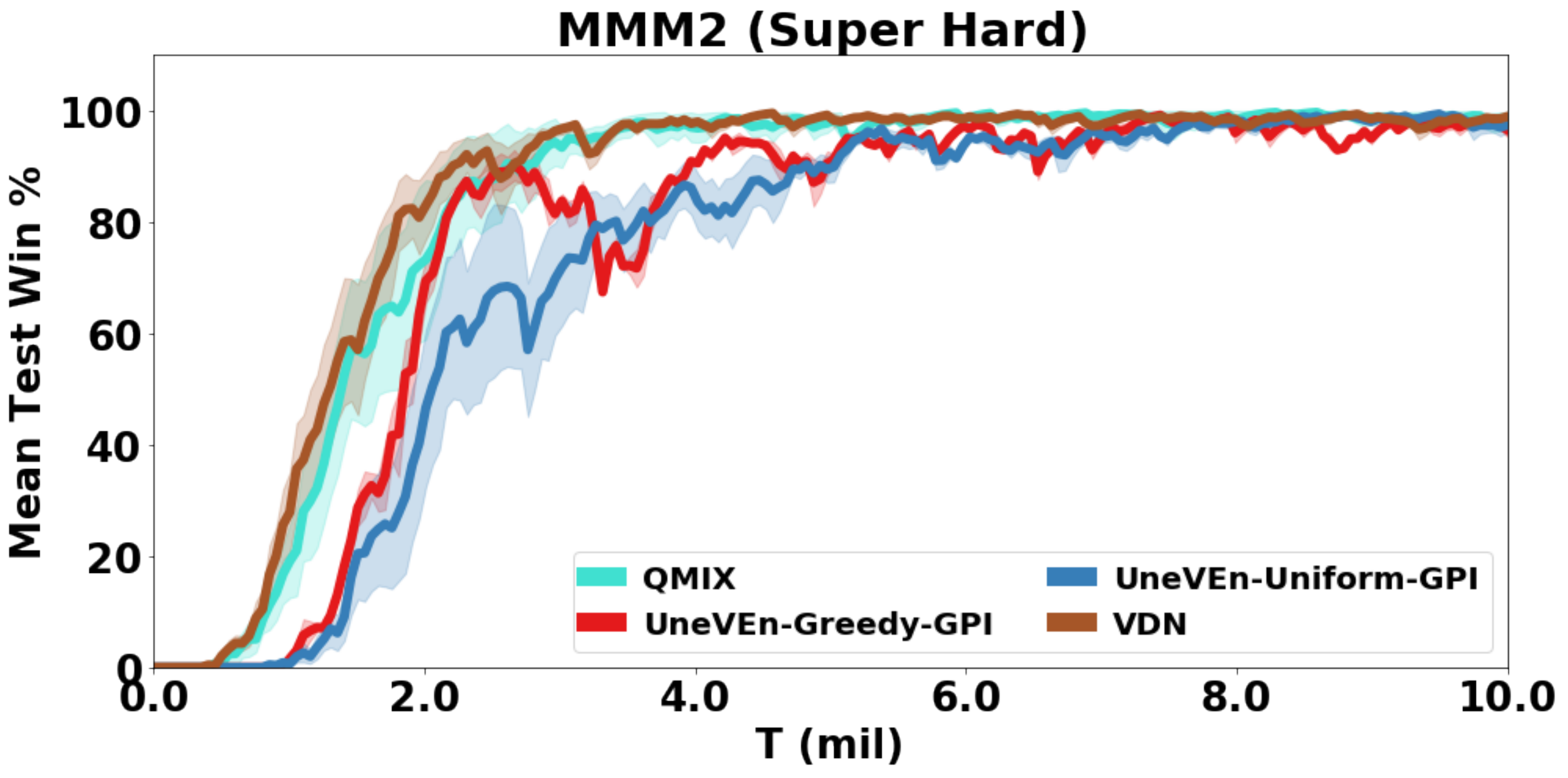}} 
\subfigure{\includegraphics[width=0.32\hsize]{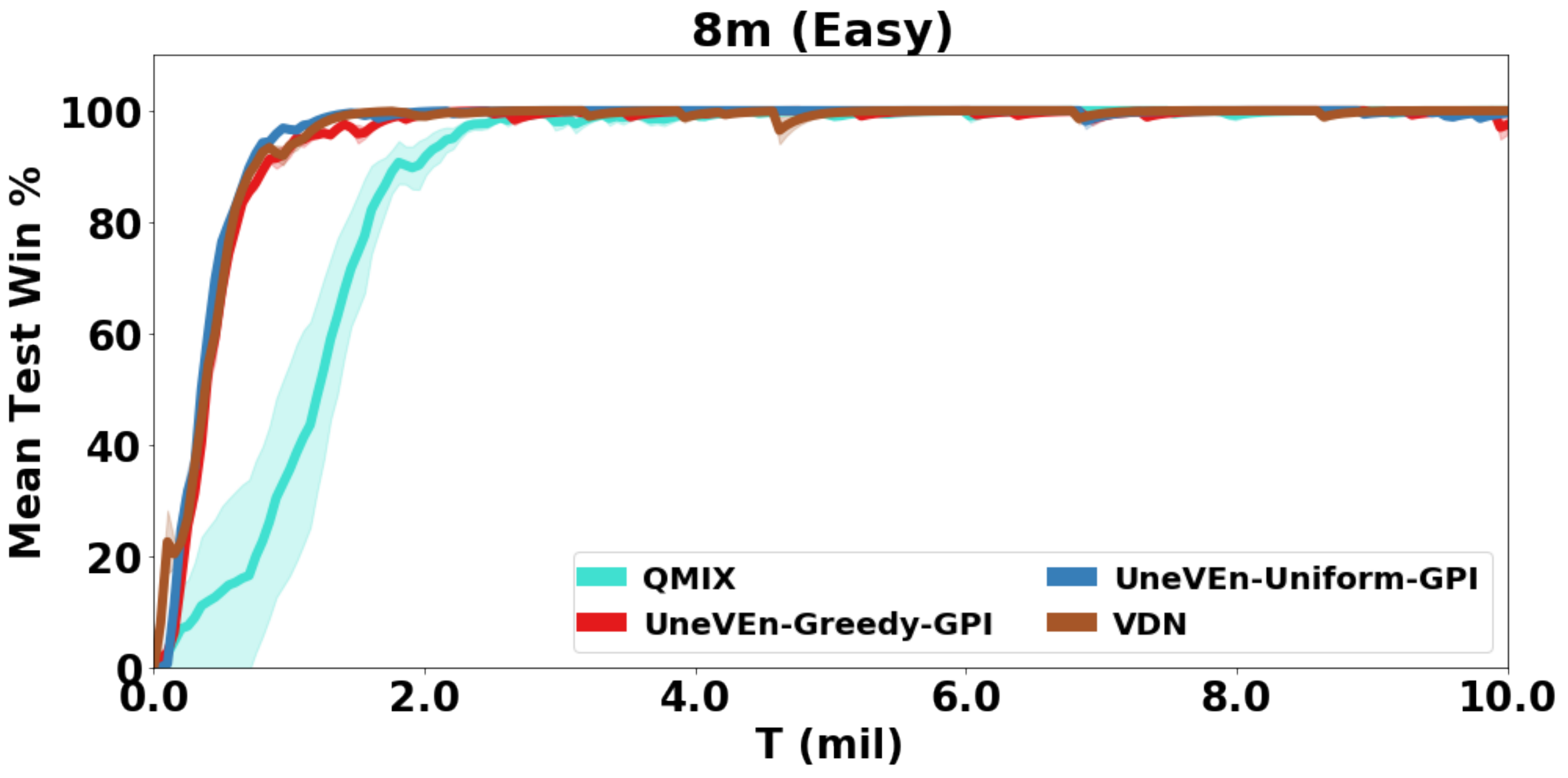}}
\subfigure{\includegraphics[width=0.32\hsize]{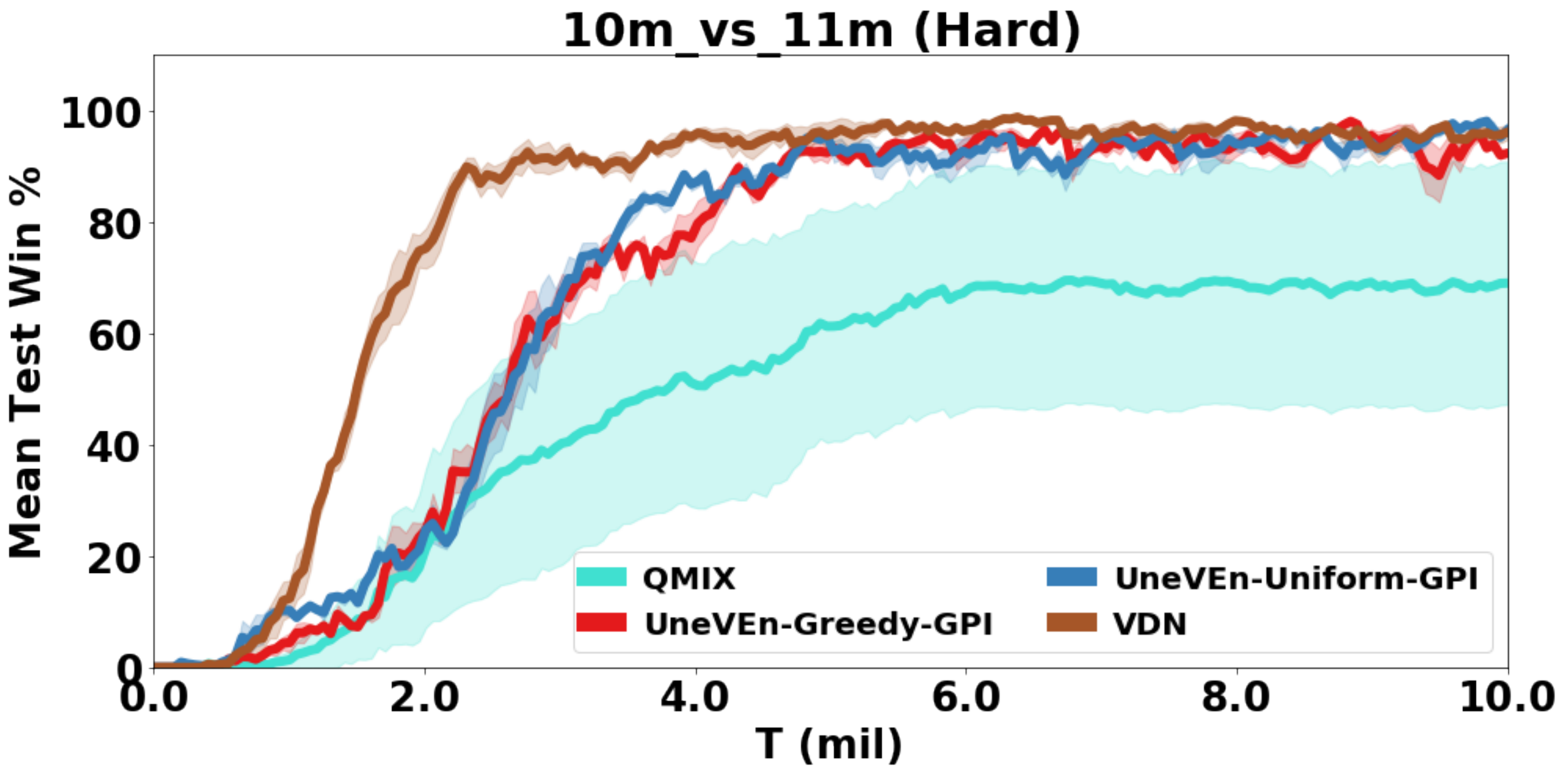}}
\caption{\small Comparison between \shorttitle, VDN and QMIX on SMAC maps with penalty $p=0.5$.}
\label{smac:all2}
\end{figure}

We first discuss different reward functions in the SMAC benchmark \citep{samvelyan2019starcraft}. The default SMAC reward function depends on three major components: (1) \texttt{delta\_enemy}: accumulates difference in health and shield of all enemy units between last time step and current time step, (2) \texttt{delta\_ally}: accumulates difference in health and shield of all ally units between last time step and current time step scaled by \texttt{reward\_negative\_scale}, (3) \texttt{delta\_deaths}: defined below.

For the original reward function \texttt{reward\_only\_positive = True}, \texttt{delta\_deaths} is defined as positive reward of \texttt{reward\_death\_value} for every enemy unit killed. The final reward is \texttt{abs}(\texttt{delta\_enemy} + \texttt{delta\_deaths}). Notice that some of the units have shield regeneration capabilities and therefore, \texttt{delta\_enemy} might contain negative values as current time step health of enemy unit might be higher than last time step. Therefore, to enforce positive rewards, \texttt{abs} function is used.

For \texttt{reward\_only\_positive = False}, \texttt{delta\_deaths} is defined as positive reward of \texttt{reward\_death\_value} for every enemy unit killed, and penalizes \texttt{reward\_negative\_scale * reward\_death\_value} for every ally unit killed. The final reward is simply: \texttt{delta\_enemy} + \texttt{delta\_deaths} - \texttt{delta\_ally}. 

Notice that \texttt{reward\_negative\_scale} measures the relative importance of lives of ally units compared to enemy units. For \texttt{reward\_negative\_scale = 1.0}, both enemy and ally units lives are equally valued, for \texttt{reward\_negative\_scale = 0.5}, ally units are only valued half of enemy units, and for \texttt{reward\_negative\_scale = 0.0}, ally units lives are not valued at all. However, \texttt{reward\_only\_positive = False} with \texttt{reward\_negative\_scale = 0.0} is NOT the same as setting \texttt{reward\_only\_positive = True} as the latter uses an \texttt{abs} function.

To summarize, \texttt{reward\_only\_positive} decides whether there is an additional penalty for health reduction and death of ally units, and \texttt{reward\_negative\_scale} determines the relative importance of lives of ally units when \texttt{reward\_only\_positive = False}. Figure \ref{Fig13:all} shows that for most of the maps, there is not a big performance difference for VDN and QMIX between different reward functions (original, $p=0.0$ and $p=0.5$). However, for some maps using \texttt{reward\_only\_positive = False} with either $p=0.0$ and $p=0.5$ improves performance over the original reward function. We hypothesize that the use of \texttt{abs} in the original reward function can detriment the learning of agent as it might get positive absolute reward to increase the health of enemy units.

Figure \ref{smac:all2} presents the mean test win rate for SMAC maps with \texttt{reward\_only\_positive = False} and  low penalty of $p=0.5$. Both VDN and QMIX achieve almost 100\% win rate on these maps and the additional complexity of learning MAUSFs in our approach results in slightly slower convergence. However, \shorttitle with both GPI schemes matches the performance as VDN and QMIX in all maps.

\begin{figure*}[htbp]
\centering     
\subfigure{\includegraphics[width=0.33\hsize]{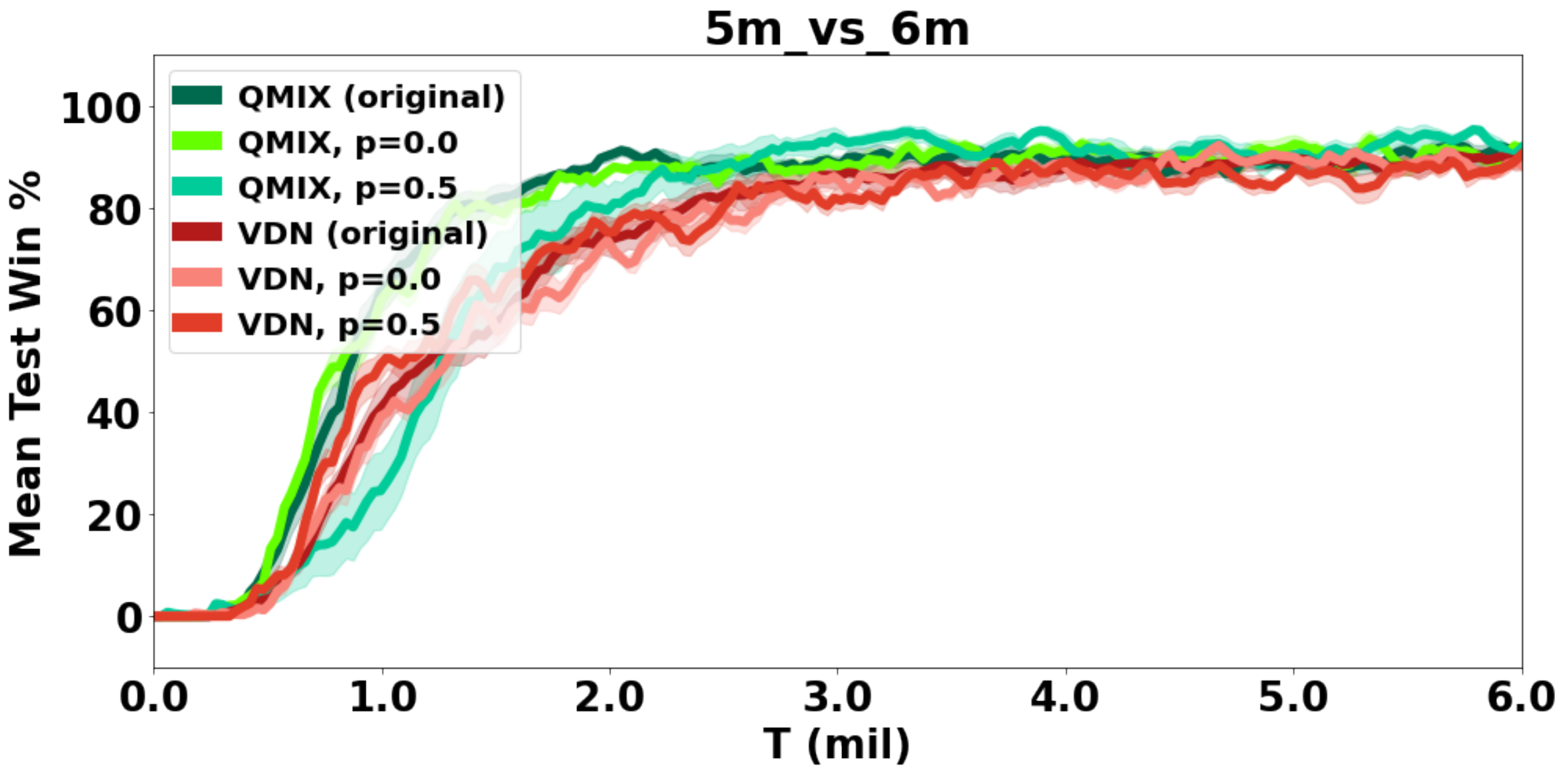}} 
\subfigure{\includegraphics[width=0.33\hsize]{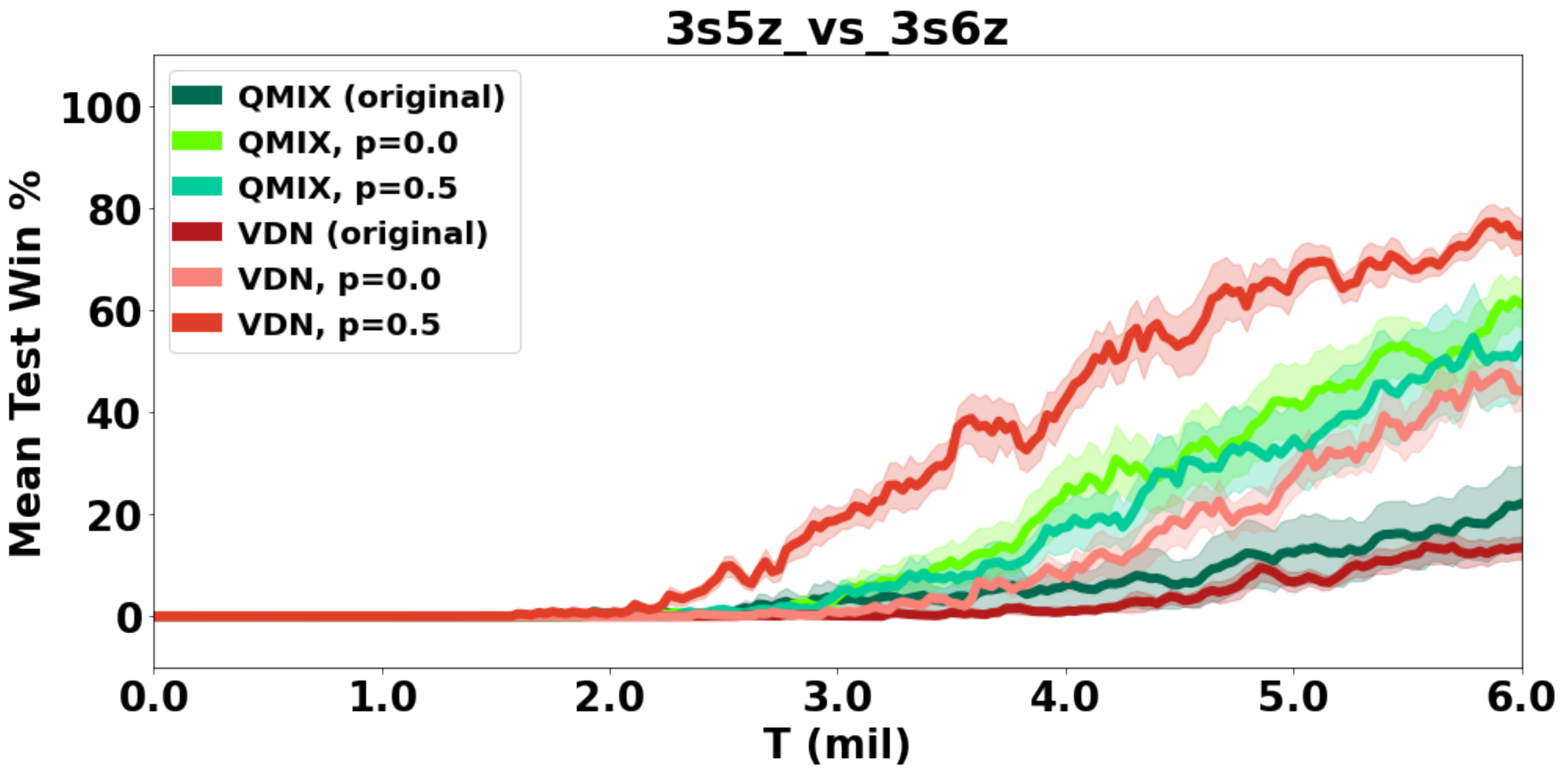}}
\subfigure{\includegraphics[width=0.33\hsize]{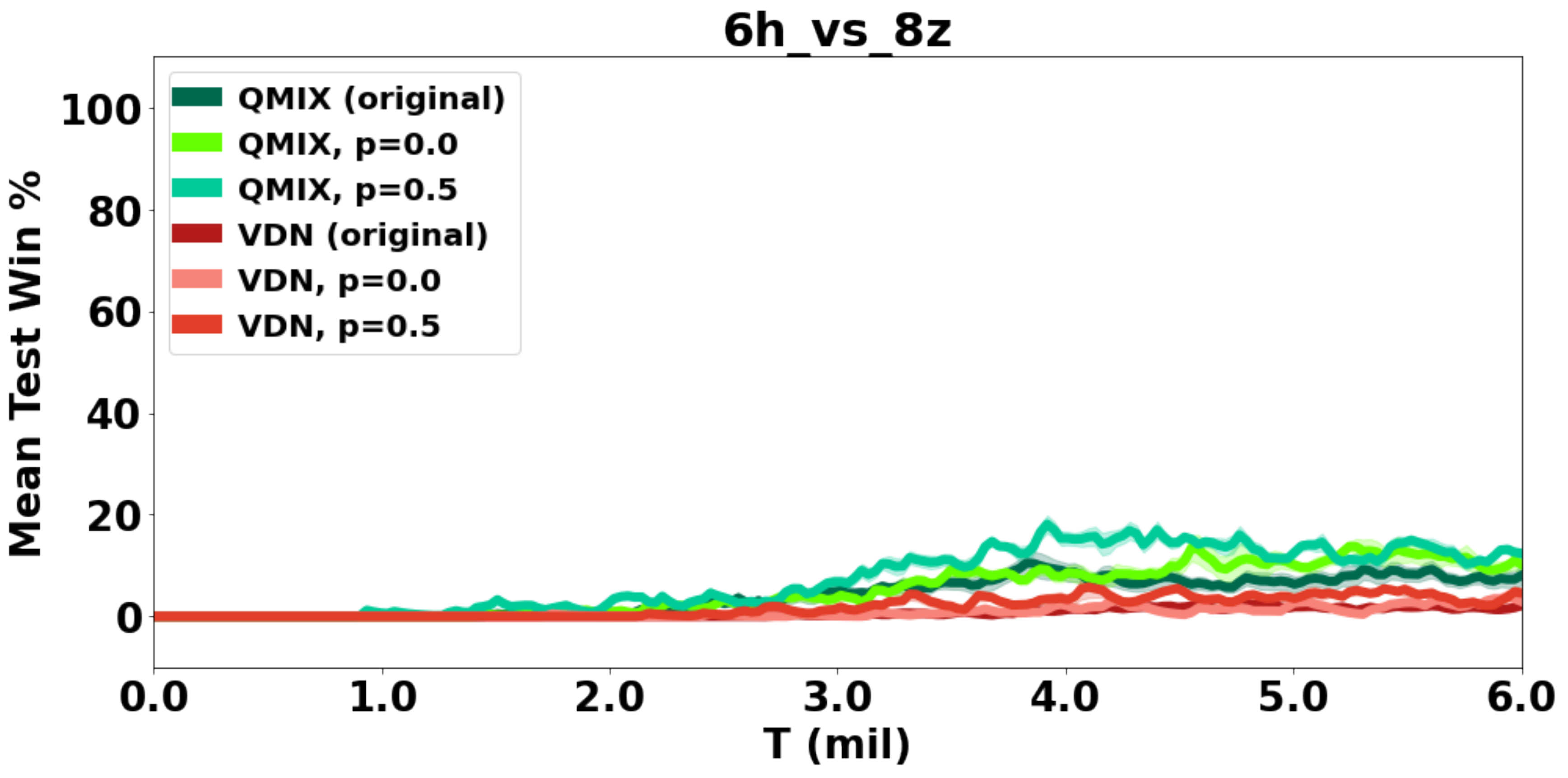}}
\par
\subfigure{\includegraphics[width=0.33\hsize]{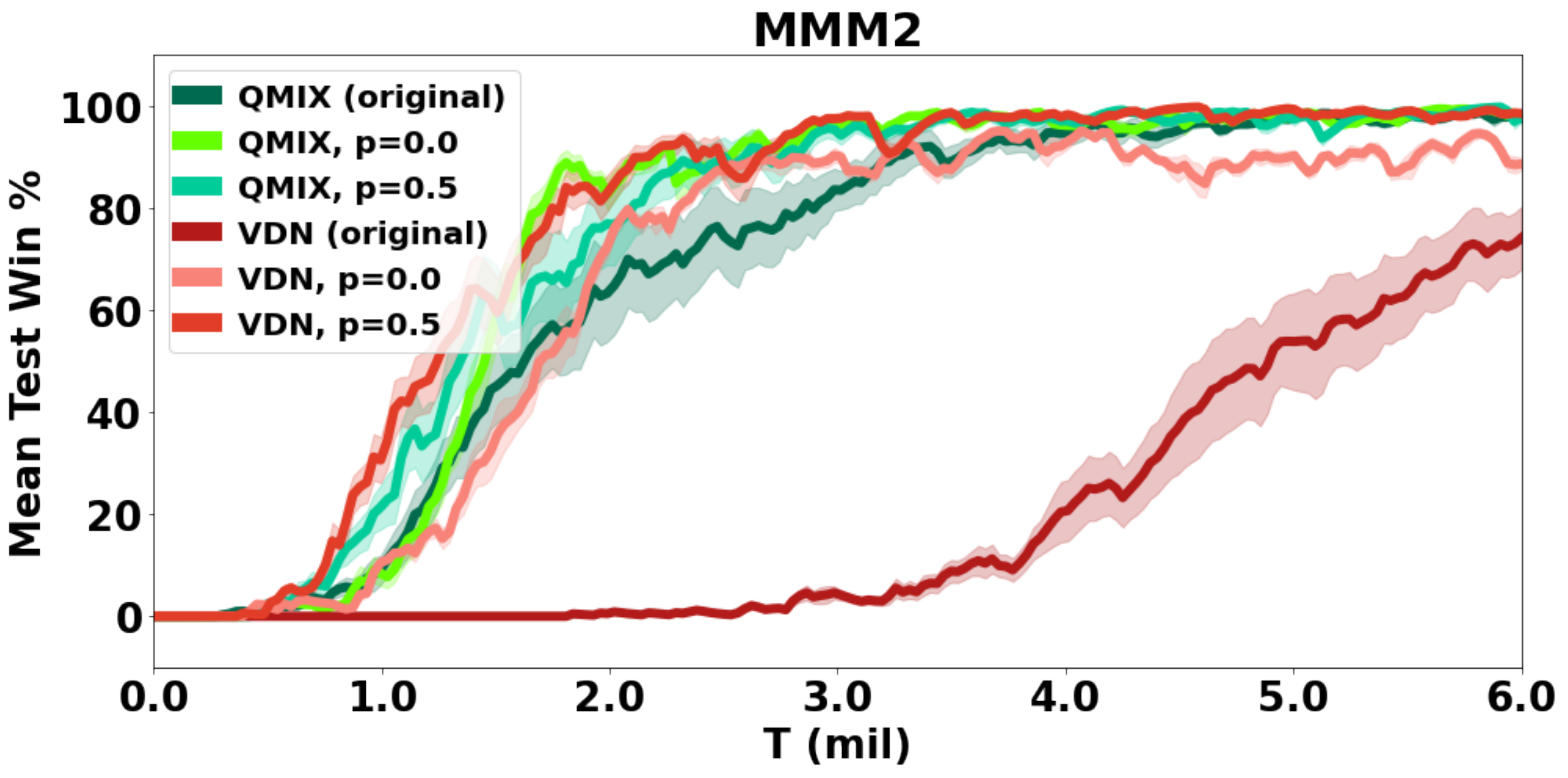}} 
\subfigure{\includegraphics[width=0.33\hsize]{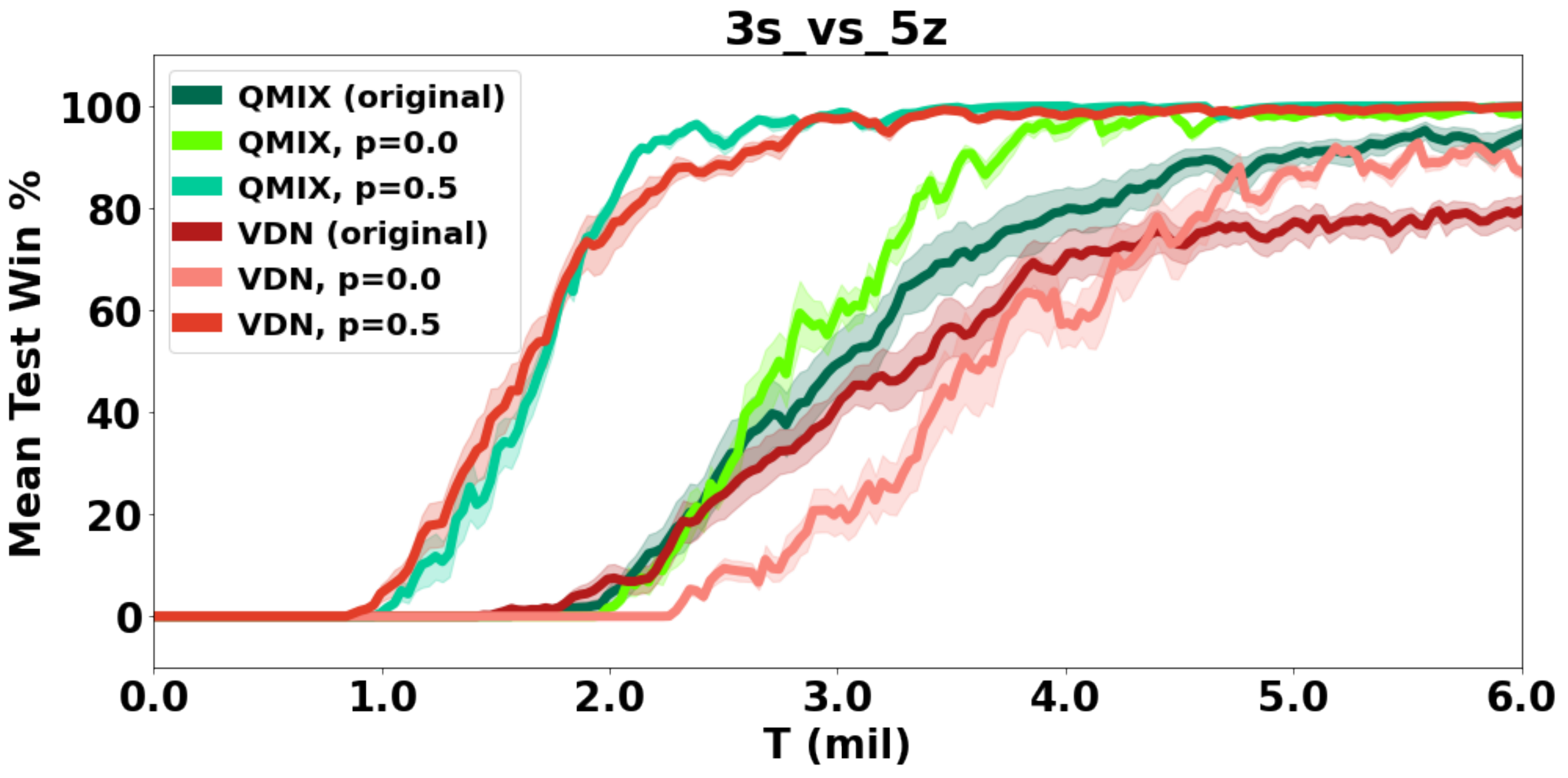}}
\subfigure{\includegraphics[width=0.33\hsize]{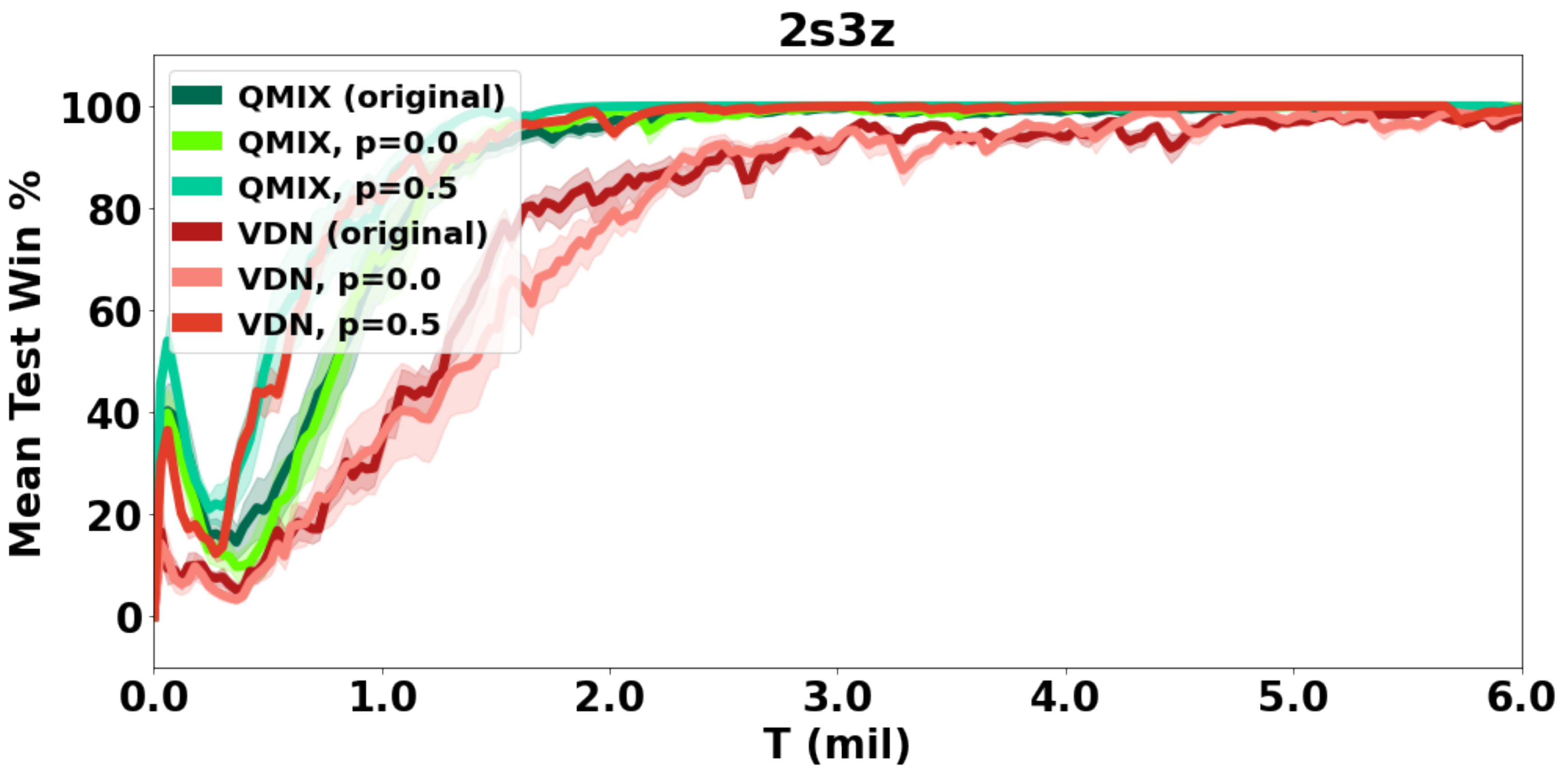}}
\par
\subfigure{\includegraphics[width=0.33\hsize]{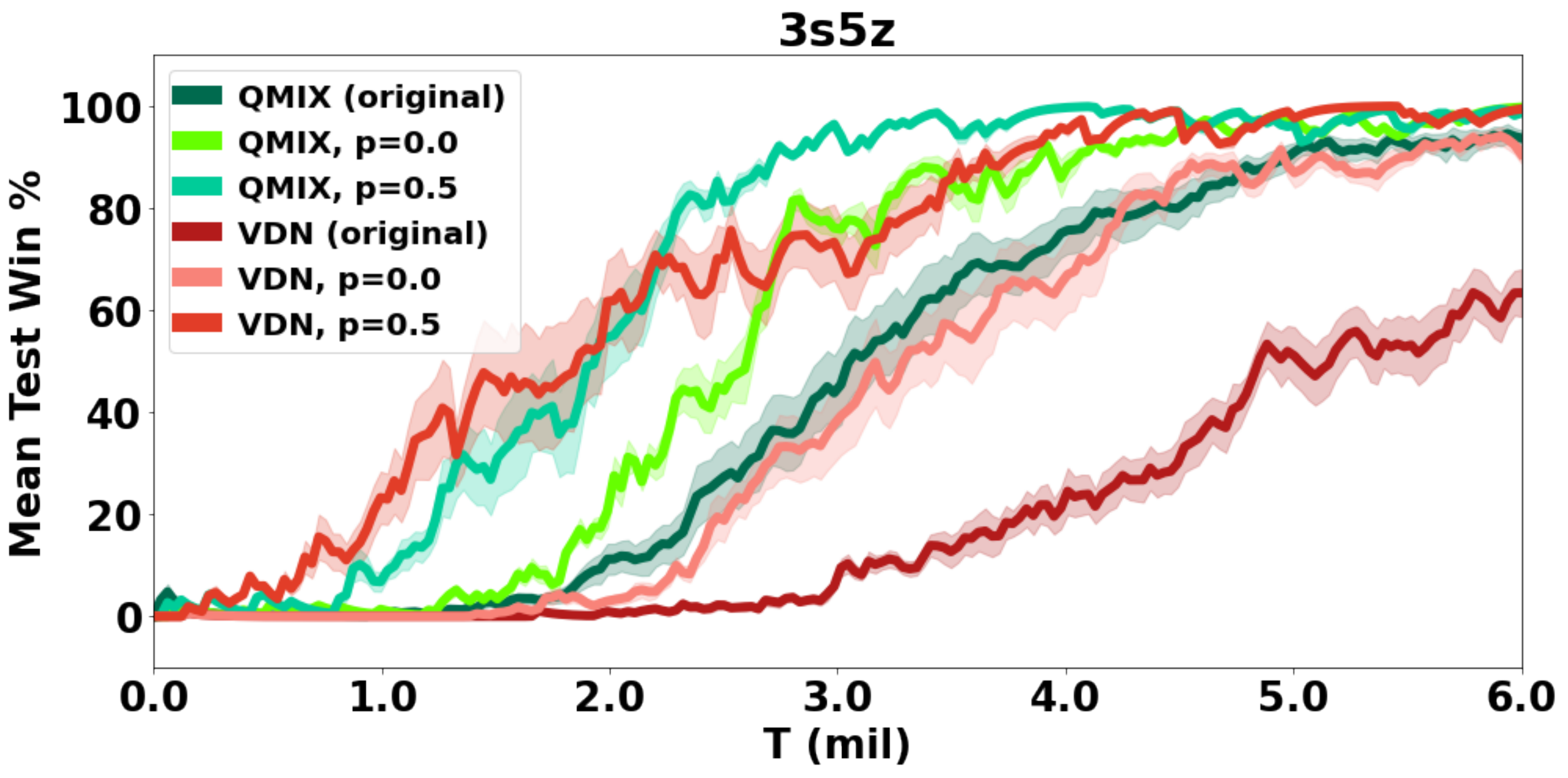}} 
\subfigure{\includegraphics[width=0.33\hsize]{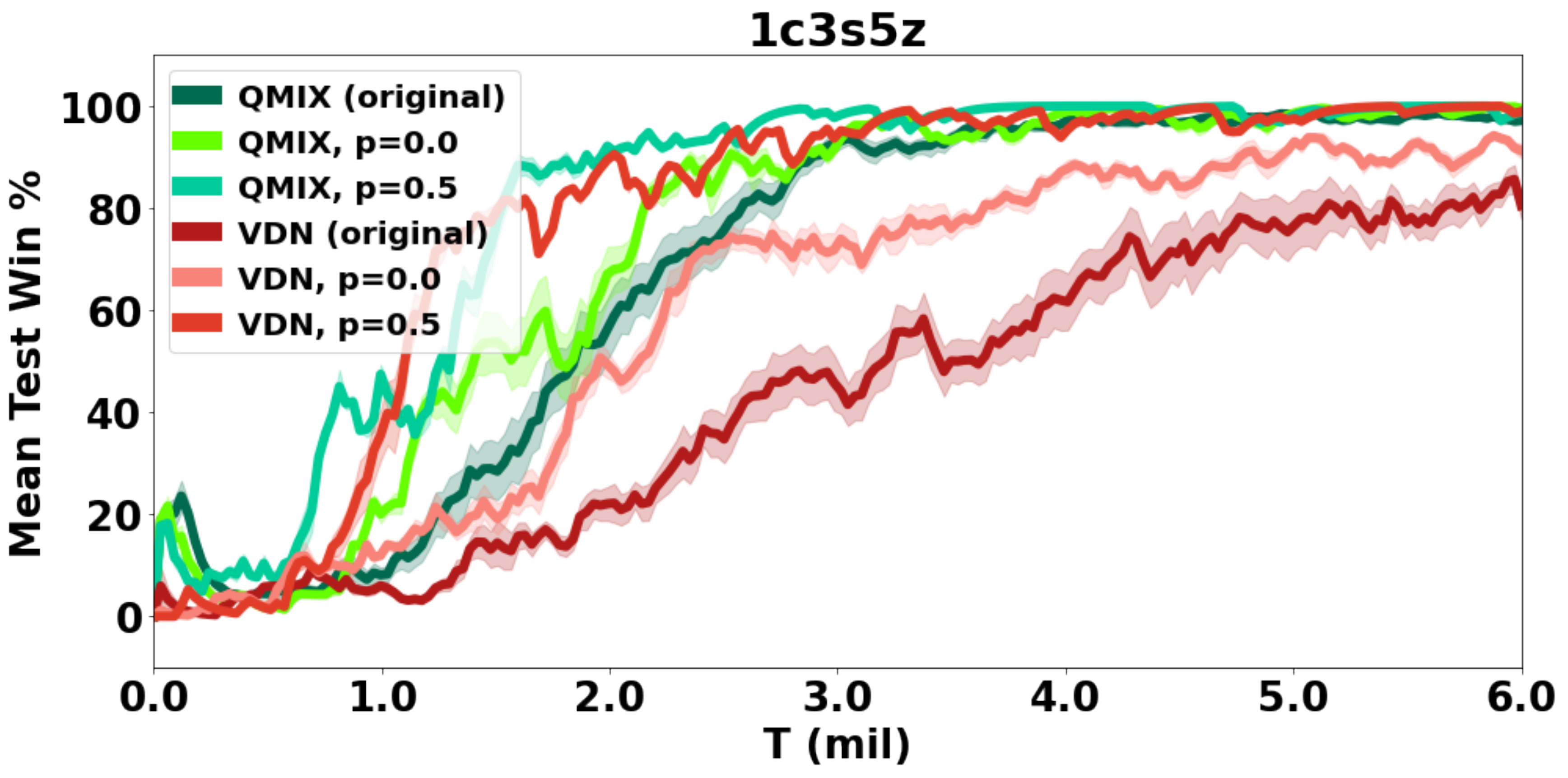}}
\subfigure{\includegraphics[width=0.33\hsize]{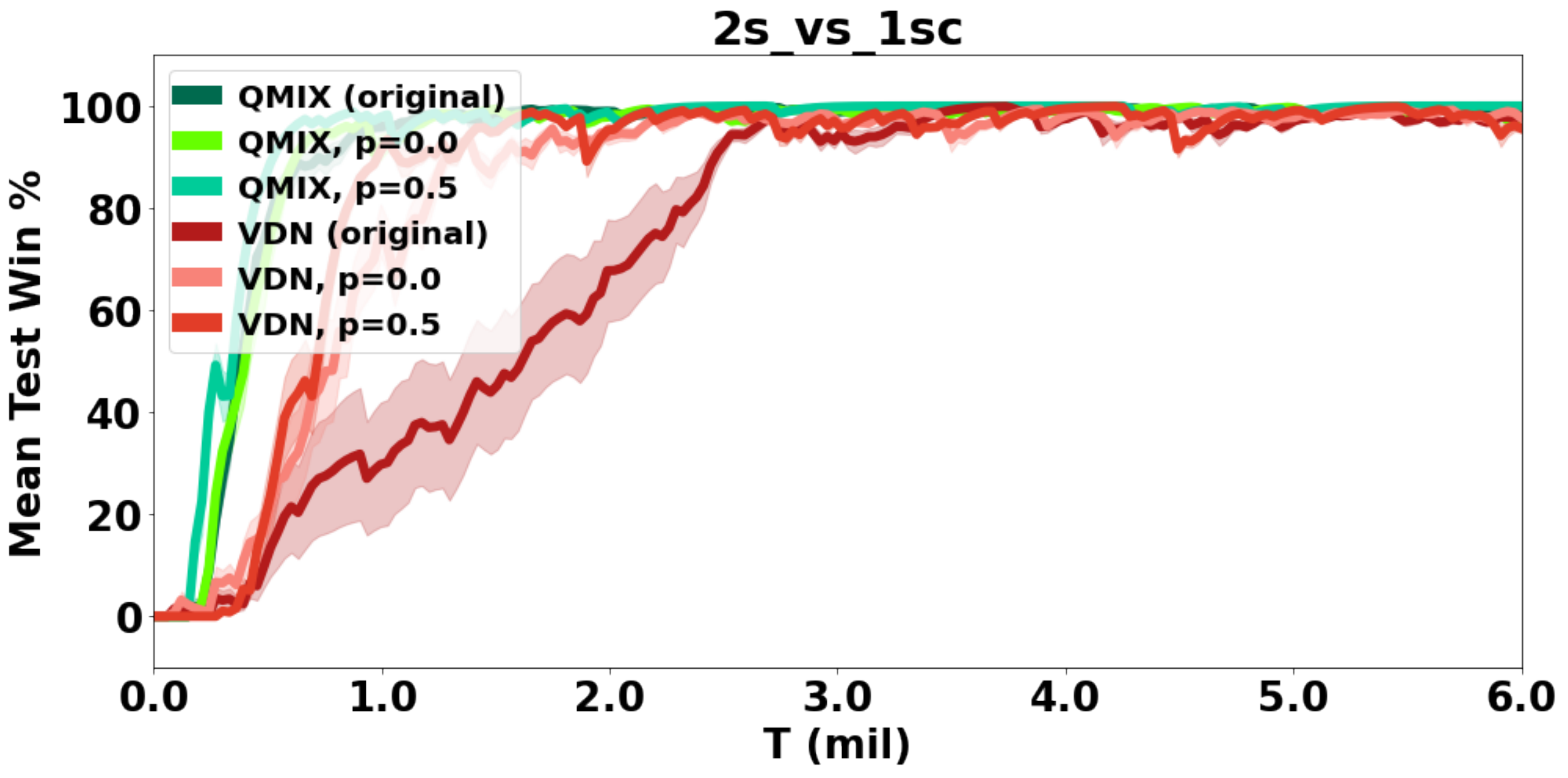}}
\par
\subfigure{\includegraphics[width=0.33\hsize]{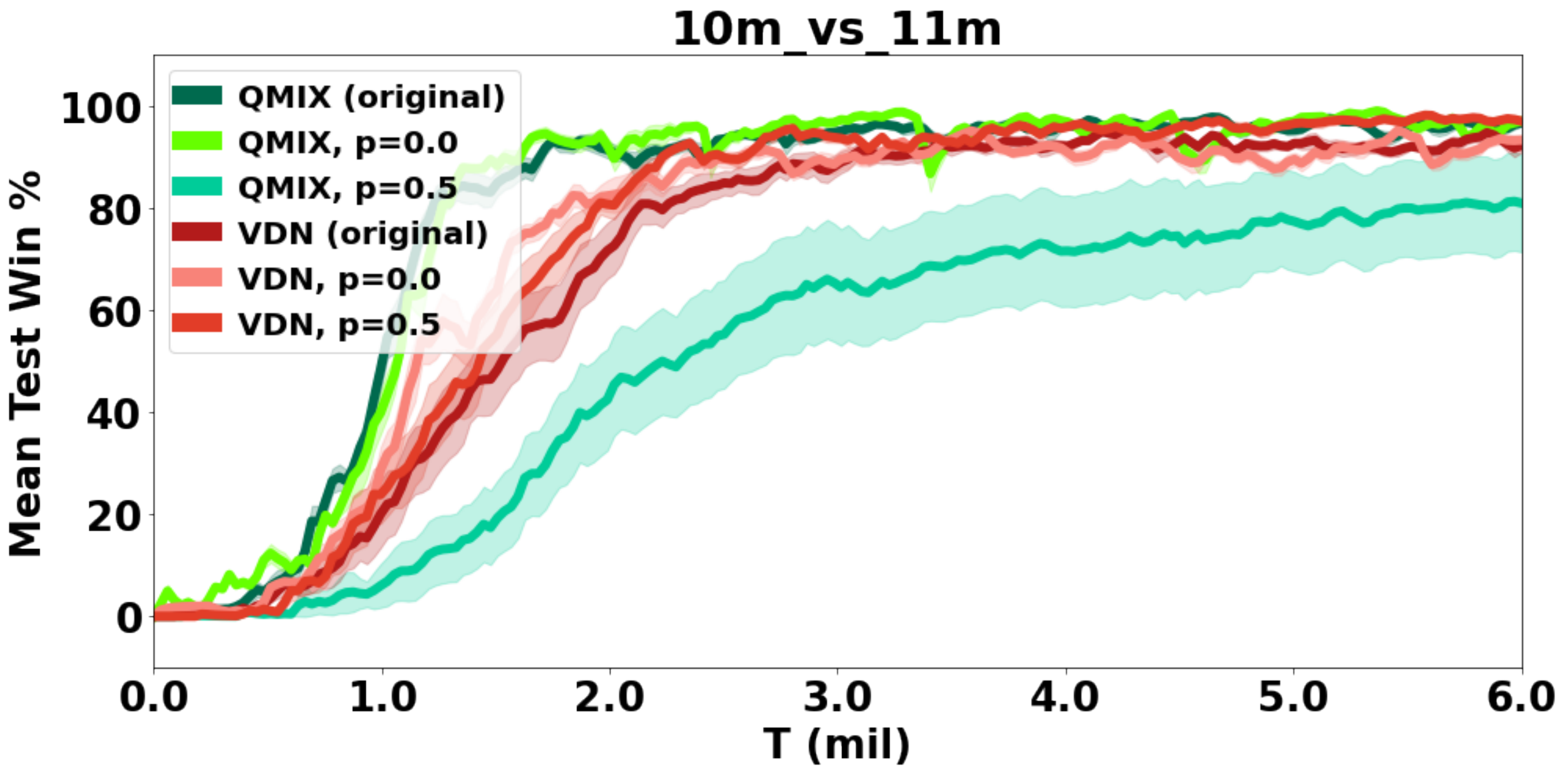}} 
\subfigure{\includegraphics[width=0.33\hsize]{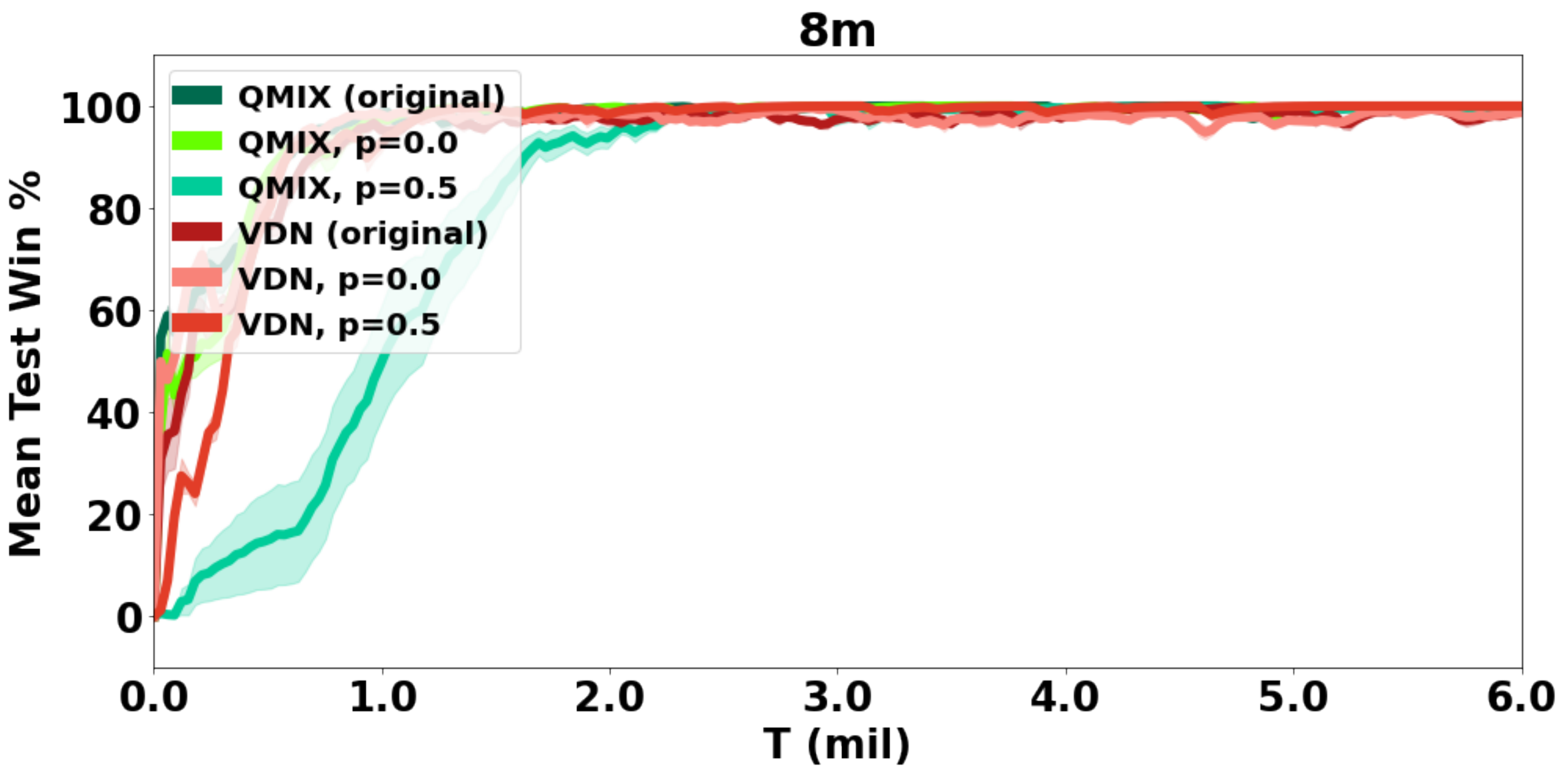}}
\subfigure{\includegraphics[width=0.33\hsize]{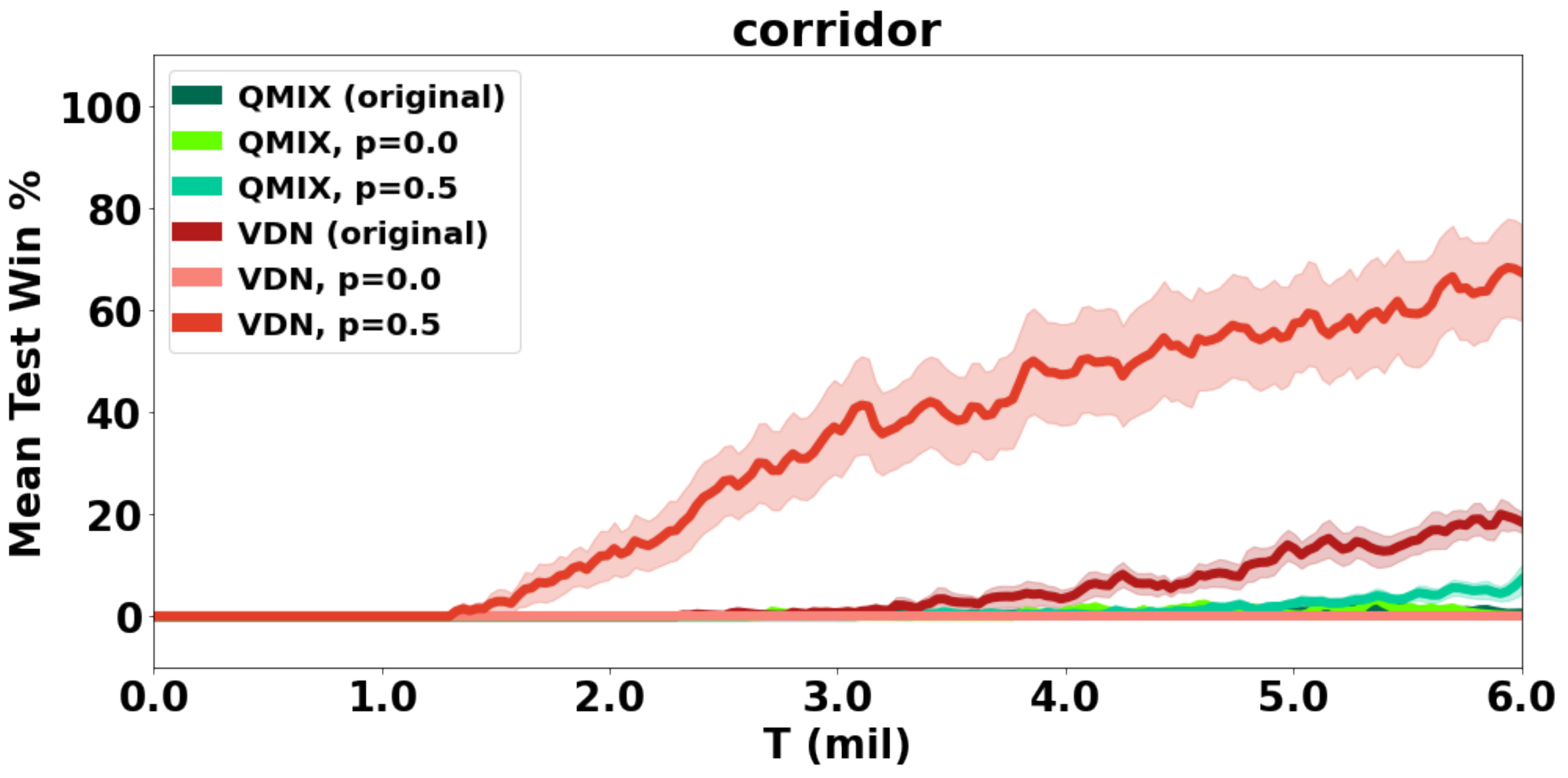}}
\par
\subfigure{\includegraphics[width=0.33\hsize]{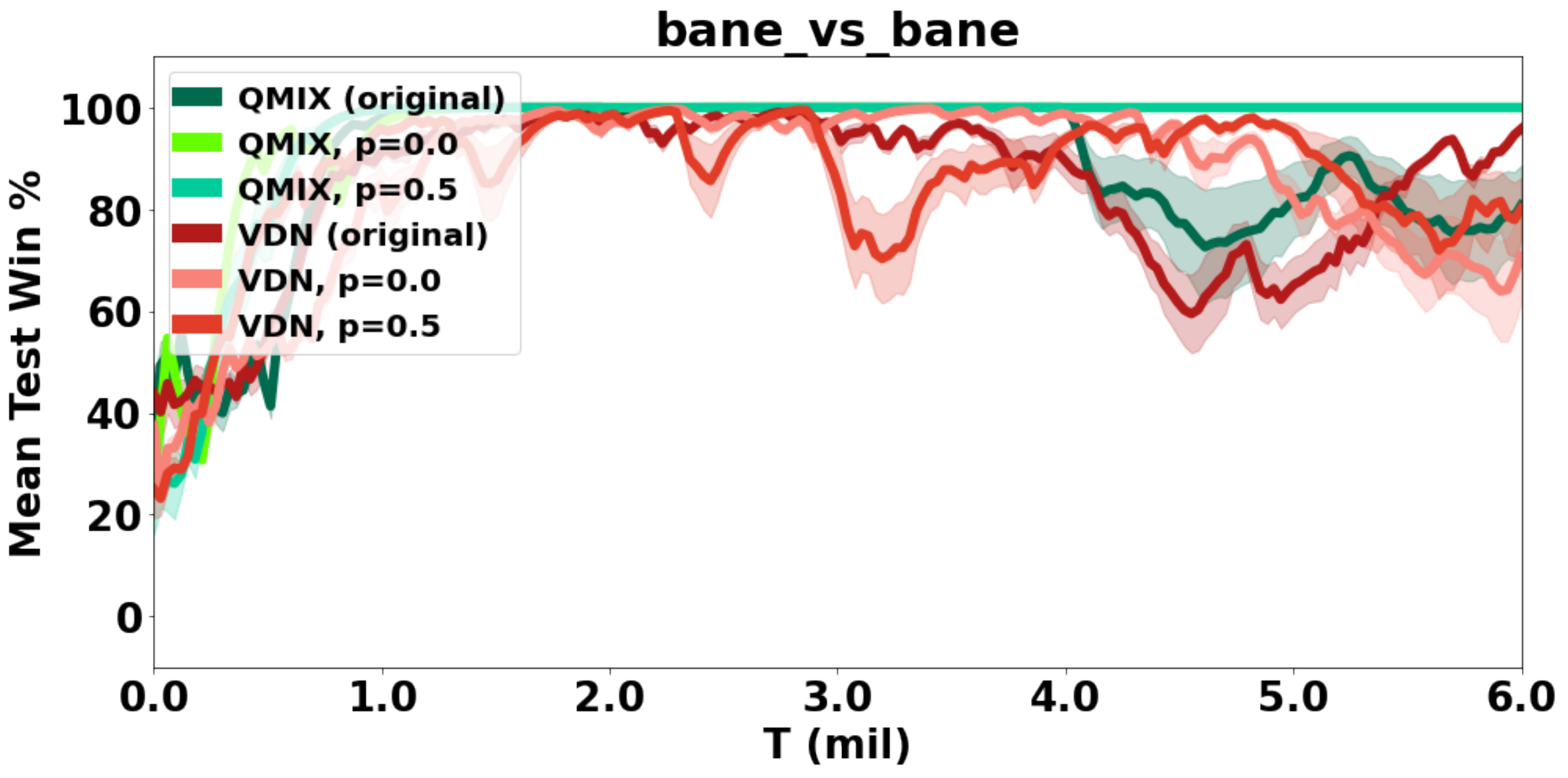}} 
\subfigure{\includegraphics[width=0.33\hsize]{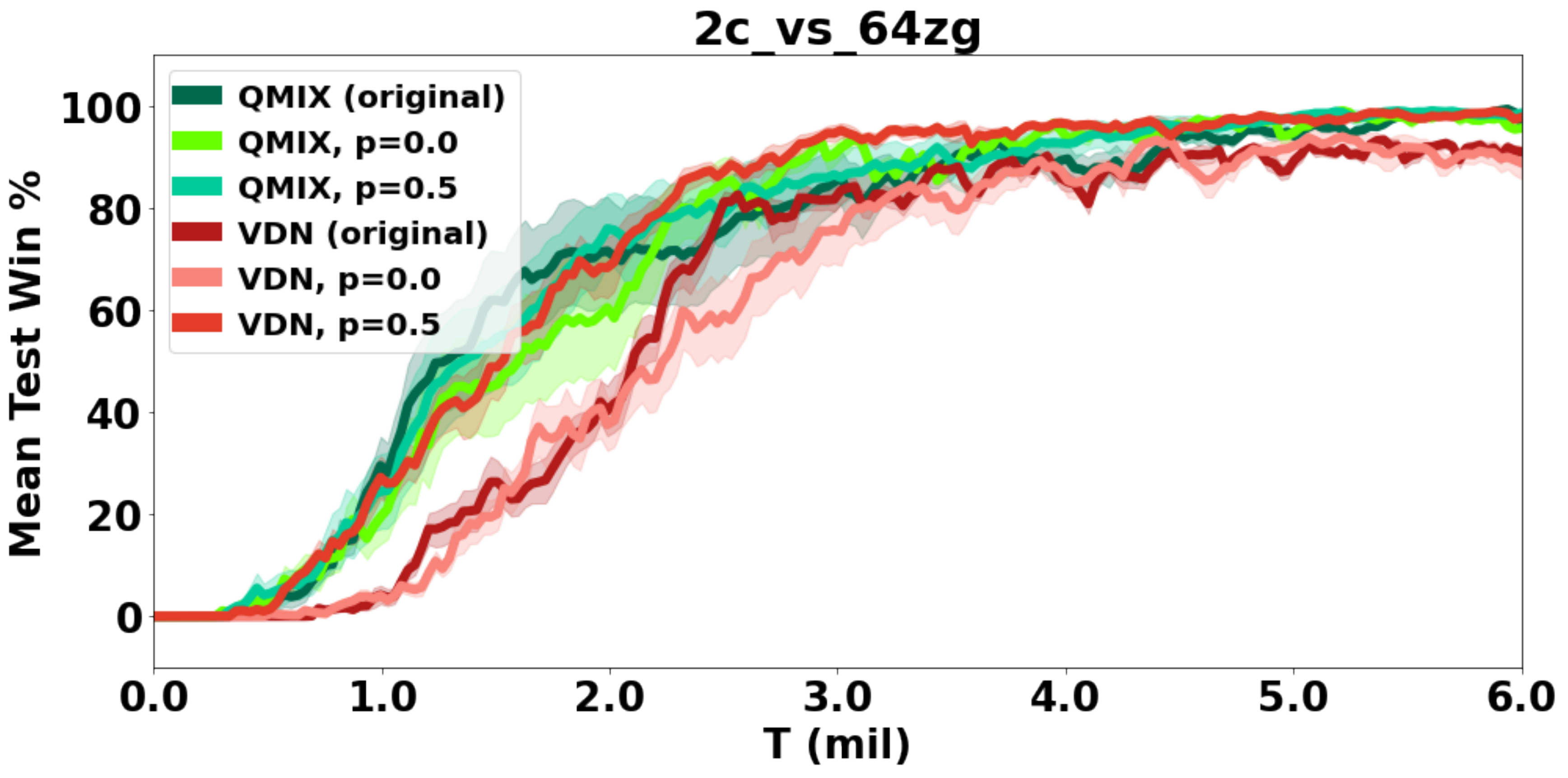}}
\subfigure{\includegraphics[width=0.33\hsize]{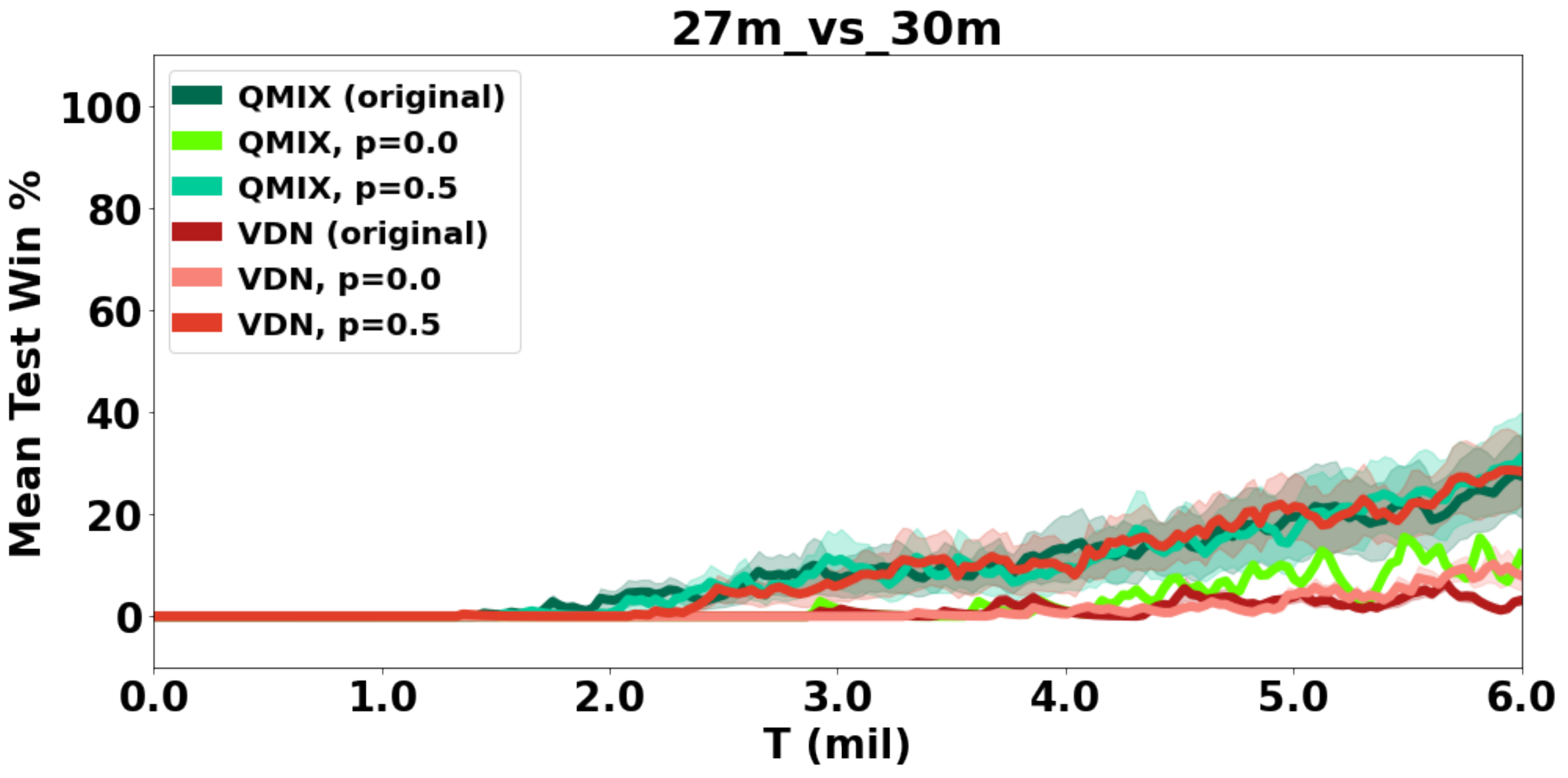}}
\caption{Comparison between VDN and QMIX baselines with original positive reward function and modified reward function with $p \in \{0.0, 0.5\}$ on SMAC maps.}
\label{Fig13:all}
\end{figure*}

%
%
%

\end{document}